\documentclass[twoside]{article}

\usepackage[accepted]{aistats2025}

\usepackage[round]{natbib}

\usepackage{amsmath,amsfonts,bm}

\def\eqref#1{equation~\ref{#1}}

\def\1{\bm{1}}

\DeclareMathAlphabet{\mathsfit}{\encodingdefault}{\sfdefault}{m}{sl}
\SetMathAlphabet{\mathsfit}{bold}{\encodingdefault}{\sfdefault}{bx}{n}

\usepackage{tikz}

\usepackage{hyperref}
\usepackage{url}
\usepackage{xcolor}
\usepackage{subcaption}
\usepackage{graphicx}
\usepackage{booktabs}
\usepackage{amsmath}
\usepackage{amsfonts}
\usepackage{algorithm}
\usepackage{algorithmicx}
\usepackage{algpseudocode}
\usepackage{mathrsfs}
\usepackage{enumitem}
\usepackage{array} 
\usepackage{amsmath}
\usepackage{amssymb}
\usepackage{amsthm}
\usepackage{amsfonts}
\usepackage{colortbl}
\usepackage{subfiles}
\usepackage{wrapfig}
\usepackage{enumitem}
\usepackage[export]{adjustbox}
\usepackage{dsfont}
\usepackage{thm-restate}

\newcommand{\mbf}{\mathbf}

\newcommand{\mscr}{\mathscr}

\newcommand\given[1][]{\:#1\vert\:}

\newcommand{\scope}{\text{sc}}

\newcommand{\pc}{\mscr{A}}
\newcommand{\pcnode}{t}

\newcommand{\pcfunc}{p}
\newcommand{\pcweight}{w}
\newcommand{\pcleaffn}{f}

\newcommand{\vtree}{V}
\newcommand{\vtreenode}{v}

\newcommand{\prods}{\text{prod}}

\renewcommand{\mbf}{\bm}

\newcommand*{\belowrulesepcolor}[1]{%
  \noalign{%
    \kern-\belowrulesep
    \begingroup
      \color{#1}%
      \hrule height\belowrulesep
    \endgroup
  }%
}
\newcommand*{\aboverulesepcolor}[1]{%
  \noalign{%
    \begingroup
      \color{#1}%
      \hrule height\aboverulesep
    \endgroup
    \kern-\aboverulesep
  }%
}

\theoremstyle{plain}
\newtheorem{theorem}{Theorem}[section]
\newtheorem{proposition}[theorem]{Proposition}
\newtheorem{claim}[theorem]{Claim}

\theoremstyle{definition}
\newtheorem{definition}[theorem]{Definition}

\theoremstyle{remark}
\newtheorem{remark}[theorem]{Remark}

\newcommand{\chg}[1]{{#1}}

\begin{document}

\pagenumbering{arabic}

\runningauthor{Honghua Zhang, Benjie Wang, Marcelo Arenas, Guy Van den Broeck}

\twocolumn[

\aistatstitle{Restructuring Tractable Probabilistic Circuits}

\aistatsauthor{Honghua Zhang* \And Benjie Wang*}
\aistatsaddress{University of California, Los Angeles \And University of California, Los Angeles}
\aistatsauthor{Marcelo Arenas \And Guy Van den Broeck}
\aistatsaddress{Pontificia Universidad Católica de Chile \And University of California, Los Angeles} 
]

\begin{abstract}
Probabilistic circuits (PCs) are a unifying representation for probabilistic models that support tractable inference. Numerous applications of PCs like controllable text generation depend on the ability to efficiently \emph{multiply} two circuits. Existing multiplication algorithms require that the circuits respect the same \emph{structure}, i.e. variable scopes decomposes according to the same \emph{vtree}. 
In this work, we propose and study the task of \emph{restructuring} structured(-decomposable) PCs, that is, transforming a structured PC such that it conforms to a target vtree. 
We propose a generic approach for this problem and show that it leads to novel polynomial-time algorithms for multiplying circuits respecting \emph{different} vtrees, as well as a practical depth-reduction algorithm that preserves structured decomposibility.
Our work opens up new avenues for tractable PC inference, suggesting the possibility of training with less restrictive PC structures while enabling efficient inference by changing their structures at inference time.
\end{abstract}

\section{INTRODUCTION}
A key challenge in deep generative modeling is the intractability of probabilistic reasoning~\citep{roth1996hardness, GehEMNLP24}. 
To address this challenge, probabilistic circuits~(PCs)~\citep{darwiche2003differential,poon2011sum,ProbCirc20} has emerged as a unifying representation of \emph{tractable} generative models for high-dimensional probability distributions. In particular, PCs support efficient and exact evaluation of various inference queries such as marginalization. The tractability of PCs has proven crucial in various applications, such as causal inference \citep{zevcevic2021interventional, wang2023compositional,busch2024psi}, knowledge graph learning \citep{loconte2023turn} and ensuring fairness in decision making~\citep{choi2021group}.

Probabilistic circuits represent distributions as \emph{computation graphs} of sums and products. A crucial aspect to the design of PCs is the \emph{structure} of the computation graph, that is, how distributions are factorized into (conditionally) independent components. The structure of PCs affects their tractability, modeling performance and computational efficiency. 
In this work, we consider the problem of \emph{restructuring} PCs: constructing a new PC that follows a particular (target) structure while representing the same distribution. We present a general algorithm for restructuring structured-decomposable circuits by considering their graphical model representations. Specifically, we leverage the graphical models to reason about conditional independencies and recursively construct a new PC conforming to the desired structure.

We then investigate two key applications of PC restructuring: circuit multiplication and depth reduction. Circuit multiplication is a fundamental operation used for answering various inference queries~\citep{VergariNeurIPS21}, such as conditioning on logical constraints~\citep{ChoiIJCAI15, ahmed2022semantic,LiuICLR24,ZhangICML23,zhang2024adaptable}, computing expected predictions of classifiers~\citep{KhosraviNeurIPS19} and causal backdoor adjustment~\citep{wang2023compositional}, as well as in improving the expressive power of circuits through squaring~\citep{loconte2024subtractive, loconte2024sum, WangTPM24}. 
Though the problem of multiplying circuits of different structures is in general \#P-hard~\citep{VergariNeurIPS21}, we identify a new class of PCs, which we call \emph{contiguous} circuits, where it is possible to multiply circuits of different structures in polynomial (or quasi-polynomial) time using our algorithm.

We also consider depth reduction, a well-established theoretical tool for reducing the depth of a circuit~\citep{valiant1983fast,raz2008balancing}. Recent PC implementations have focused on layer-wise parallelization of PC inference via modern GPUs, and depth reduction enables greater parallelization~\citep{PeharzICML20,DangAAAI21,LiuICML24,loconte2024relationship}. In this work, we show that our restructuring algorithm can be used to transform a structured-decomposable circuit to an equivalent log-depth circuit, with much tighter upper bounds than given by prior work. This opens up new possibilities of practically implementing depth reduction techniques to speed up PC inference.

In summary, our key contributions are: (1) we identify the \emph{restructuring} problem, which concerns the computational complexity of converting between different tractable PC structures; (2) we present a general algorithm for mapping between any two structures (Section \ref{sec:restructuring}); (3) we show that restructuring is possible in (quasi-)polynomial complexity for a class of \emph{contiguous} structures (Section \ref{sec:multiplication});  and (4) we illustrate the application of restructuring to two important problems: \emph{multiplying} probability distributions (Section \ref{sec:multiplication}) and reducing the \emph{depth} of the model structure (Section \ref{sec:depth_reduction}), where we derive novel results on tractability.

\section{PROBABILISTIC CIRCUITS}
\paragraph{Notation} We will use uppercase to denote variables (e.g. $X$) and lowercase to denote values of those variables (e.g. $x$). We use boldface to denote sets of variables/values (e.g. $\mbf{X}, \mbf{x}$).

\begin{definition}[Probabilistic Circuit]
    A probabilistic circuit (PC) $\pc = (\mathcal{G}, \mbf{w})$ represents a joint probability distribution over random variables $\bm{X}$ through a rooted directed acyclic (computation) graph (DAG), consisting of sum ($\oplus$), product ($\otimes$), and leaf nodes ($L$), parameterized by $\mbf{w}$. Each node $\pcnode$ represents a probability distribution $\pcfunc_{\pcnode}(\bm{X})$, defined recursively by:    
    \begin{align*}
        \pcfunc_{\pcnode}(\bm{x}) = 
        \begin{cases}
            \pcleaffn_{\pcnode}(\bm{x}) & \text{if $t$ is a leaf node} \\
            \prod_{c \in \text{ch}(t)} \pcfunc_{c}(\bm{x}) & \text{if $t$ is a product node} \\
             \sum_{c \in \text{ch}(t)} w_{t, c} \pcfunc_{c}(\bm{x}) & \text{if $t$ is a sum node}
        \end{cases}
    \end{align*}
    where $f_t(\bm{x})$ is a univariate input distribution function (e.g. Gaussian, Categorical), we use $\text{ch}(t)$ to denote the set of children of a node $t$, and $w_{t, c}$ is the non-negative weight associated with the edge $(t, c)$ in the DAG, which satisfy the constraint that $\sum_{c \in \text{ch}(t)} w_{t, c} = 1$ for every sum node $t$.     
    We define the \emph{scope} of a node $t$ to be the variables it depends on. 
    The function represented by a PC, denoted $p_{\pc}(\mbf{x})$, is the function represented by its root node; and the size of a PC, denoted $|\pc|$, is the number of edges in its graph. 
\end{definition}

Intuitively, product nodes represent a factorized product of its child distributions, while sum nodes represent a weighted mixture of its child distributions. For simplicity, in the rest of this paper we assume that sum/leaf and product nodes alternate (i.e. child of a sum is a product, and child of a product is a leaf or sum), and that each product has exactly two children.
The key feature of PCs is their \emph{tractability}, i.e., the ability to answer queries about the distributions they represent exactly and in polynomial time. Two commonly assumed properties known as smoothness and decomposability ensure efficient marginalization:

\begin{definition}[Smoothness and Decomposability]
    A sum node is \emph{smooth} if all of its children have the same scope. A product node is \emph{decomposable} if its children have disjoint scope. A PC is smooth (resp. decomposable) if all of its sum (resp. product) nodes are smooth (resp. decomposable).
\end{definition}

Intuitively, decomposability requires that a product node partitions its scope among its children. For many other important queries, it is useful to enforce a stronger form of decomposability, known as \emph{structured-decomposability}, that requires that product nodes with the same scope decompose in the same way.

\begin{definition}[Vtree]
    A vtree $V$ over variables $\mbf{X}$ is a rooted binary tree, where each $X \in \bm{X}$ is associated with a unique leaf node $v$ (we write $X_v$ for the variable associated with node $v$). Each inner node $v$ covers a set of variables $\mbf{X}_v$, satisfying $\mbf{X}_v = \mbf{X}_l \cup \mbf{X}_r$ where $l, r$ are the children of $v$. We write $V_v$ to denote the subtree rooted at $v$.
\end{definition}

\begin{definition}[Structured Decomposability]
    A PC $\pc$ is structured-decomposable (w.r.t a vtree $V$) if every product node $t \in \pc$ decomposes its scope according to some inner vtree node $v \in V$.
\end{definition}

The main advantage of structured decomposability is that it enables tractable circuit multiplication of two circuits respecting the same vtree, which is a core subroutine for many applications. 
However, structured decomposable circuits can be less expressive efficient in general \citep{de2021compilation}.

\section{PC RESTRUCTURING} \label{sec:restructuring}
In this section, we describe a generic approach that restructures any structured-decomposable PC to respect a target vtree. The approach consists of three steps: (1) construct a Bayesian network representation of the PC; (2) find sets of latent variables in the Bayesian network that induce conditional independecies required by the target vtree; (3) construct a new structured PC recursively leveraging the conditional independence derived in (2).

\subsection{Structured PCs as Bayesian Networks}
\label{subsec:pc_to_graphical}
It is known that one can efficiently compile a tree-shaped Bayesian network to an equivalent probabilistic circuit~\citep{darwiche2003differential, poon2011sum, dang2020strudel, LiuNeurIPS21}.
In this subsection, we describe how to go in the opposite direction, i.e. converting an arbitrary structured-decomposable PC to a tree-shaped Bayesian network with linearly many variables.

Let $\pc$ be a structured PC over variables $\mbf{X}$ respecting vtree $\vtree$.
Given a vtree node $v \in \vtree$, we write $\prods(\vtreenode)$ to denote the set of all product nodes with scope $\mbf{X}_{\vtreenode}$. We define the \emph{hidden state size} $h$ of the circuit to be $\max_{v \in \vtree} |\prods(\vtreenode)|$. Writing $n$ for the number of variables, the size of the circuit is then $O(nh^2)$.\footnote{The number of active sum nodes per vtree node is at most $h$, as each such node must have a different product node parent corresponding to the parent vtree node scope. This leads to $O(h^2)$ edges per vtree node.}

We begin by providing a latent variable interpretation of structured PCs. Specifically, we define an augmented PC which explicitly associates latent variables with product nodes for each variable scope. Given some vtree node $v$, let us associate each $t \in \prods(v)$ with a unique index $\text{idx}(t) \in \{0, ..., |\prods(v)| - 1\}$, also writing $t_{v, i}$ to refer to the product node with index $i$ in $\prods(v)$. Then we can introduce a categorical latent variable $Z_v$ whose value corresponds to a particular product node in $\prods(v)$:

\begin{definition}[Augmented PC]
    Given a structured-decomposable and smooth PC $\pc$ over variables $\mbf{X}$ respecting vtree $V$, we define the augmented PC $\pc_{\text{aug}}$ to be a copy of $\pc$ where for each vtree node $v \in V$, we add an additional child $t_{\text{aug}}$ to each product node $t \in \prods(v)$ that is a leaf node with scope $Z_v$ and leaf function $\pcleaffn_{t_{\text{aug}}}(Z_v) = \mathds{1}_{Z_v = \text{idx}(t)}$.
\end{definition}

It is not hard to see that the augmented PC $\pc_{\text{aug}}$ is a PC over variables $\mbf{X}, \mbf{Z}$ and retains structured decomposability and smoothness. Further, the standard marginalization algorithm for PCs ensures that the augmented PC has the correct distribution:

\begin{restatable}{proposition}{propAug}
    $p_{\pc}(\mbf{X}) = \sum_{\mbf{z}} p_{\pc_{\text{aug}}}(\mbf{X}, \mbf{z})$
\end{restatable}

Let $\vtree_{\vtreenode \to Z_{\vtreenode}}$ be the rooted DAG obtained by replacing all inner nodes $\vtreenode$ in vtree $\vtree$ with variable $Z_v$ (cf. Fig. \ref{fig:one}). Now, we claim that the augmented PC can be interpreted as a Bayesian network with graph structure $\vtree_{\vtreenode \to Z_{\vtreenode}}$. To do this, we construct a distribution $p^{*}(\mbf{X}, \mbf{Z})$, based on the augmented PC, that factorizes as required by the Bayesian network structure. 
\begin{figure}
    \centering
    \begin{subfigure}[b]{0.49\linewidth}
    \centering
    \includegraphics[width=0.85\linewidth]{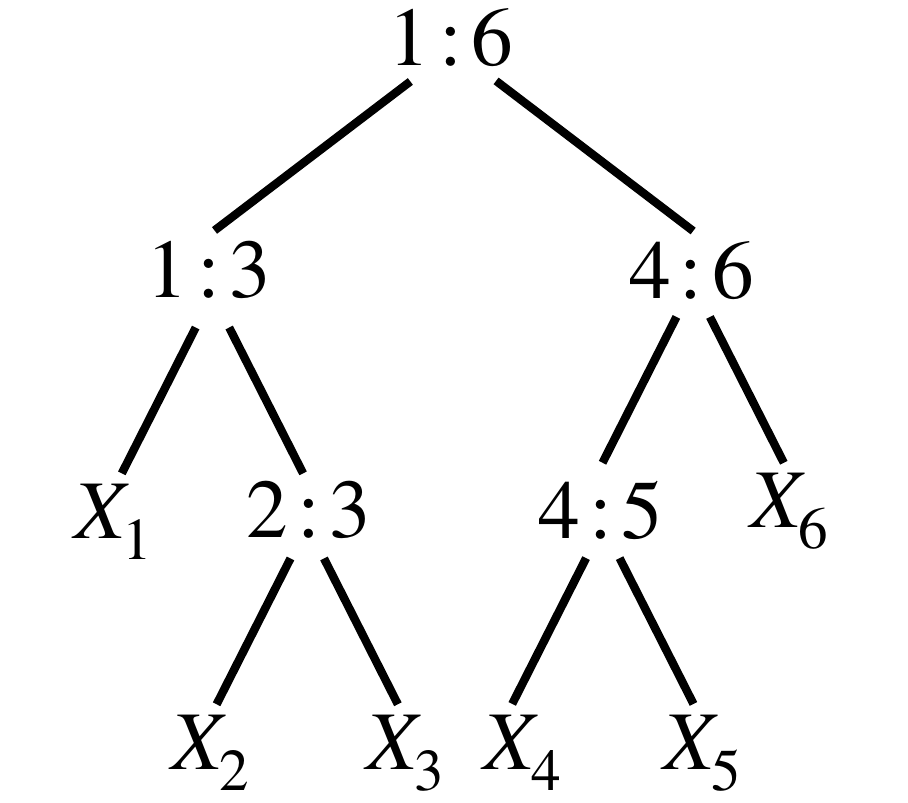}
    \caption{A (contiguous) vtree $V$}
    \label{fig:vtree}
    \end{subfigure}
    \begin{subfigure}[b]{0.49\linewidth}
    \centering
    \includegraphics[width=0.84\linewidth]{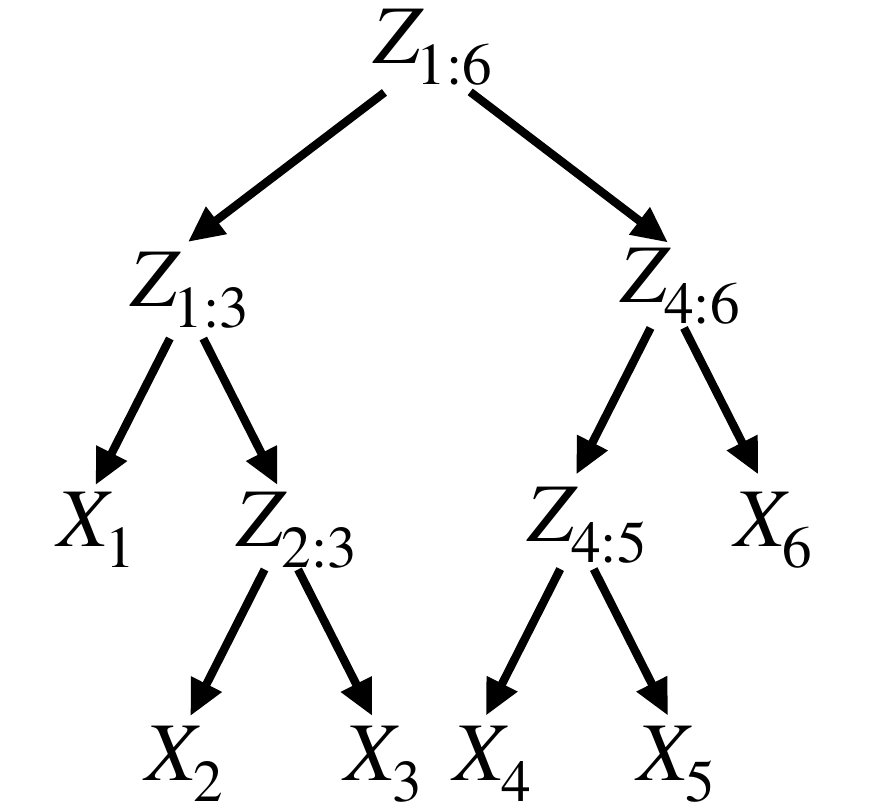}
    \caption{Bayesian network $V_{v\mapsto Z_v}$}
    \label{fig:bn}
    \end{subfigure}\\
    \hfill \\
    \begin{subfigure}[b]{0.95\linewidth}
    \centering
    \includegraphics[width=0.9\linewidth]{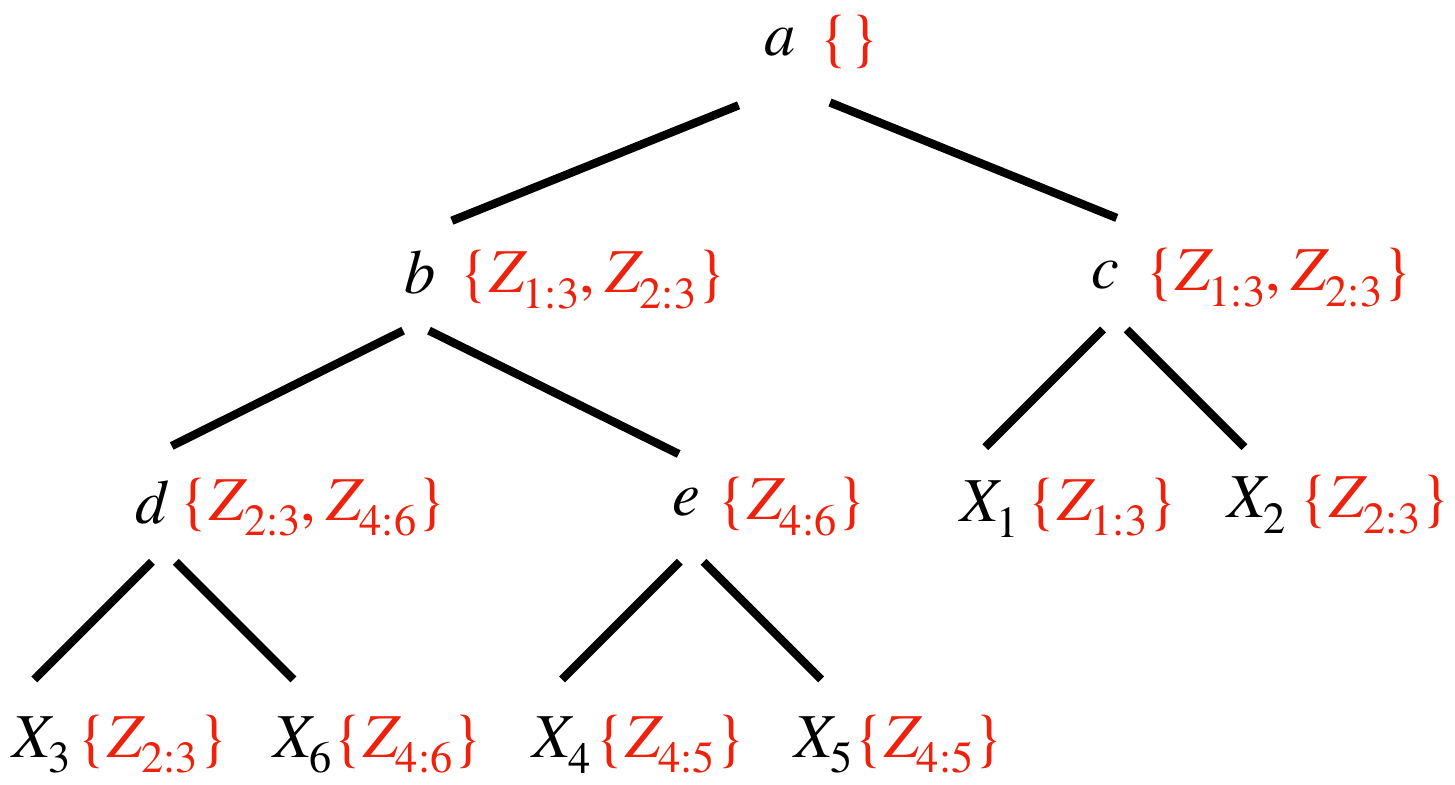}
    \caption{A \textcolor{red}{labelling} of vtree $W$.}
    \label{fig:labelling}
    \end{subfigure}
    \caption{Fig.~\ref{fig:vtree} shows a vtree $V$ for some contiguous PC $\mscr{A}$; Fig.~\ref{fig:bn} shows a Bayesian network representation $G_{\mscr{A}}$ for $\mscr{A}$; Fig.~\ref{fig:labelling} shows a valid labelling of vtree $W$ with respect to $G_{\mscr{A}}$.}
    \label{fig:one}
\end{figure}
There are three cases to consider: (i) the root node $p^{*}(Z_{\textsc{root}(V)
})$, 
(ii) the leaf nodes $p^{*}(X_{v}|Z_p)$, and (iii) other nodes $p^{*}(Z_v|Z_p)$ (where we write $p$ for the parent of $v$ in $V$). 
In case (i), we set $p^{*}(Z_v = i) := w_i$ where $w_i$ is the weight of the edge from the root sum node to the product node $t_{v, i}$. In case (ii), we set $p^{*}(X_{v}|Z_p = j) = p_t(X_{v})$, where $t$ is the leaf node child (with scope $X_{v}$) of the product node $t_{p, j}$.
Finally, in case (iii) we note that due to alternating sums and products, $t_{p, j}$ must have a sum node child, which may or may not have a weighted edge to $t_{v, i}$ (whose weight we denote by $w_{ij}$ if it exists). We thus define:
\begin{align*}
    p^{*}(Z_v = i| Z_p = j) = 
    \begin{cases}
        w_{ij} & \exists \text{ path from } t_{p, j} \text{ to } t_{v, i} \\
        0 & \text{otherwise}
    \end{cases}
\end{align*}

It remains to show that this distribution faithfully represents the distribution of the augmented PC, i.e. $p_{\pc_{\text{aug}}} = p^{*}$. The intuitive idea is that each value of $\mbf{Z}$ corresponds to a subtree of $\pc_{\text{aug}}$, whose value is precisely given by the product of weights and leaf functions specified by the Bayesian network; we refer readers to the Appendix for the complete proof. We thus have the following mapping from structured PCs to tree-shaped Bayesian networks:
\begin{restatable}{theorem}{thmPCBN}
    Let $\pc$ be a structured-decomposable and smooth PC over variables $\mbf{X}$ respecting vtree $V$. Then there exists a Bayesian network $G_{\mscr{A}}$ over variables $\mbf{X}$ and $\mbf{Z} = \{Z_v| v \in V\}$ with graph $V_{v \mapsto Z_v}$ such that $\sum_{z} p_{G}(\mbf{X}, \bm{z}) = p_{\pc}(\mbf{X})$.
\end{restatable}

Since we have shown that $p_{\mscr{A}}$ and $p_{G}$ represents the same distribution over the observed variables $\mbf{X}$, we will drop the subscripts when there is no ambiguity.

\subsection{Recursive PC Restructuring}
Suppose we have a PC $\mscr{A}$ with its Bayesian network representation $G_{\mscr{A}}$ and vtree $V$, and let $W$ be some other vtree. We now show how to construct a new PC respecting $W$ that encodes the same distribution as $\mscr{A}$. The rough idea is to label each vtree node $w\!\in\!W$ with a subset of latent variables $\mbf{C}_w\!\subseteq\!G_{\mscr{A}}$ such that $\mbf{X}_w$ is conditionally independent from $\mbf{X} \setminus \mbf{X}_w$ given $\mbf{C}_{w}$. To characterize such properties, we introduce \emph{covers}:
\begin{definition}[Blocked Path]
    Given a directed rooted tree, we say that a path $P$ is blocked by a set of nodes $\mbf{S}$ if $P \cap \mbf{S}\neq\varnothing$.
\end{definition}
\begin{definition}[Cover]
Given a tree-shaped Bayesian network $G_{\mscr{A}}$ as constructed in Sec.~\ref{subsec:pc_to_graphical}, we say that $\mbf{C} \subseteq \mbf{Z}$ \emph{covers} $\mbf{S} \subseteq \mbf{X}$ if $\mbf{C}$ blocks all paths between $\mbf{S}$ and $\mbf{X}\!\setminus\!\mbf{S}$ in $G_{\mscr{A}}$.
\end{definition}
Our definition of cover is a special case of \emph{d-separation}~\citep{geiger1990d}, which characterizes conditional independence for Bayesian networks:
\begin{proposition}[\citet{geiger1990d}]
$\mbf{A}, \mbf{B}\!\subseteq\!G_{\pc}$ are conditionally independent given $\mbf{C}\!\subseteq\!G_{\mscr{A}}$ if and only if $\mbf{C}$ blocks all paths between $\mbf{A}$ and $\mbf{B}$.\footnote{Defining blocked paths for DAGs requires considering \emph{colliders}, which do not occur in directed rooted trees.} In particular, if $\mbf{C}$ covers $\mbf{S}$ then $\mbf{S}$ and $\mbf{X}\!\setminus\!\mbf{S}$ are conditionally independent given $\mbf{C}$.
\end{proposition}
Our goal is to recursively construct vectors of sum nodes $\oplus_i$ representing the probability distributions $p(\mbf{X}_{w} \given \mbf{C}_w\!=\!i)$. Letting $l$ and $r$ be the children of $w$, we will establish a recurrence relation between $p(\mbf{X}_{w} \given \mbf{C}_w)$, $p(\mbf{X}_{l} \given \mbf{C}_l)$ and $p(\mbf{X}_{r} \given \mbf{C}_r)$. This requires the vtree labels to satisfy the following properties: 
\begin{definition}[Valid Vtree Labelling] \label{def:correct_labels}
    Given the Bayesian network $G_{\pc}$ and target vtree $W$, a valid \emph{labelling} of $W$ with respect to $G_{\pc}$ associates each node $w\!\in\!W$ with a subset of latent variables $\mbf{C}_w\!\subseteq\!G_{\vtree}$ s.t.
    \begin{enumerate}[noitemsep, leftmargin=*]
        \item $\mbf{C}_w$ covers $\mbf{X}_w$ in $G_{\pc}$.
        \item $\mbf{C}_l$ blocks all paths between $\mbf{X}_l$ and $\mbf{C}_r \cup \mbf{C}_w$.
        \item $\mbf{C}_r$ blocks all paths between $\mbf{X}_r$ and $\mbf{C}_l \cup \mbf{C}_w$.
    \end{enumerate}
Furthermore, w.l.o.g., we set $\mbf{C}_{\text{root of } W}\!:=\!\varnothing$ and $\mbf{C}_{X_j}\!:=\!\text{parent of } X_j \text{ in } G_{\pc}$ for the leaf nodes $X_j\!\in\!W$. See Figure \ref{fig:labelling} for an example.
\end{definition}

\begin{figure}
    \centering
    \includegraphics[width=0.95\linewidth]{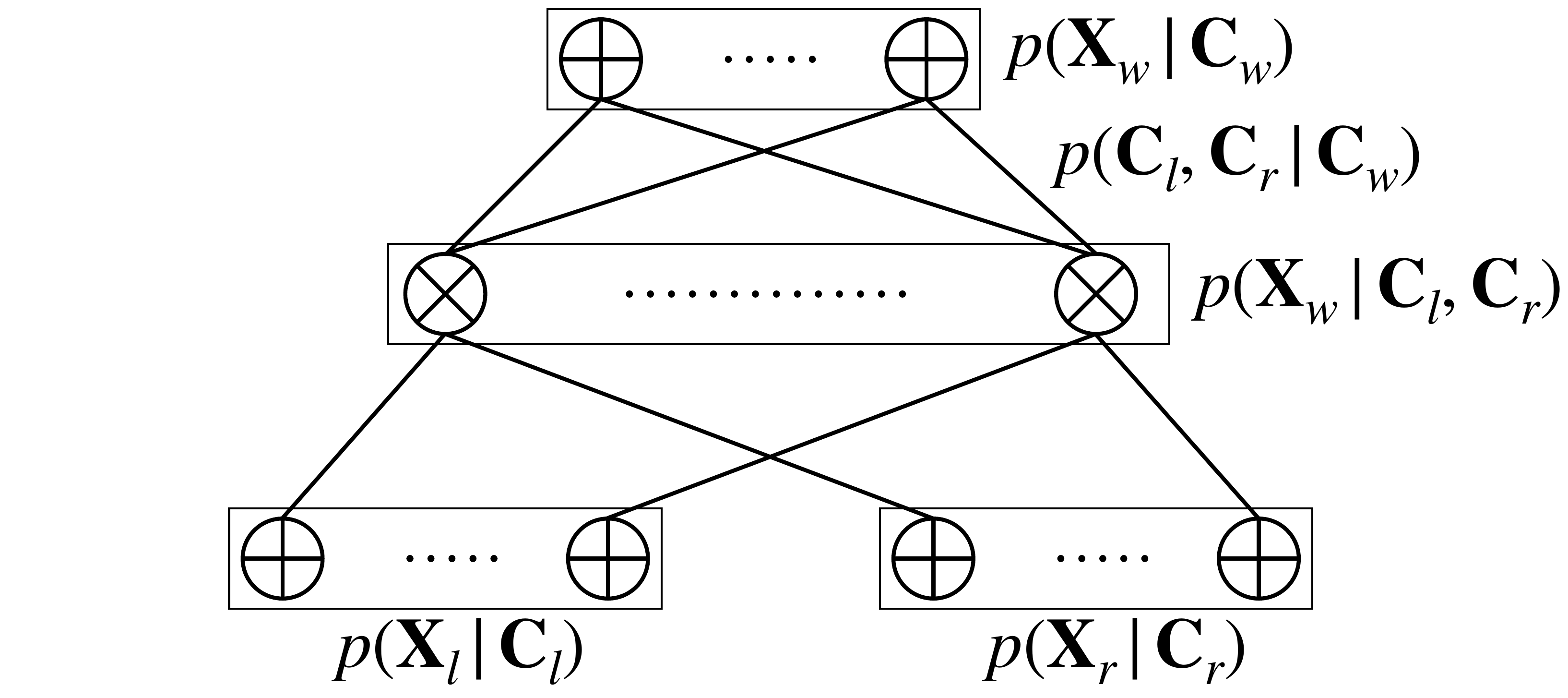}
    \caption{Recursive construction of vectors of sum nodes representing $p(\mbf{X}_w \given \mbf{C}_w)$}
    \label{fig:construction}
\end{figure}

Assuming that we have computed a valid labelling for $W$, we can then proceed to construct the desired PC by a bottom-up recursion on $W$. For the base case, if $w$ is a leaf node representing some random variable $X_j$, $p(X_j \given \mbf{C}_{X_j}) = p(X_j \given \text{parent of } X_j \text{ in } G_{\pc})$, which is directly given by the conditional probability table of $G_{\pc}$. For the induction step, when $w$ is a inner node with children $l$ and $r$, we have the recurrence relation:
\begin{align*}
&p(\mbf{X}_w \given \mbf{C}_w) \\
&= \sum_{(\mbf{C}_l \cup \mbf{C}_r) \setminus \mbf{C}_w}p(\mbf{X}_l, \mbf{X}_r \given \mbf{C}_l, \mbf{C}_r) \cdot p(\mbf{C}_l, \mbf{C}_r \given \mbf{C}_w) \\
&= \sum_{(\mbf{C}_l \cup \mbf{C}_r) \setminus \mbf{C}_w}p(\mbf{X}_l \given \mbf{C}_l) \cdot p(\mbf{X}_r \given \mbf{C}_r) \cdot p(\mbf{C}_l, \mbf{C}_r \given \mbf{C}_w)
\end{align*}
Here the first step follows from Property 2 and 3, and the second step follows from all properties in Defintion~\ref{def:correct_labels}.
The circuit materialization of the recurrence relation is shown in Figure~\ref{fig:construction}. Note that if $w$ is the root, then $p(\mbf{X}_w \given \mbf{C}_w)$ becomes $p(\mbf{X})$, which is a single sum node representing the distribution of $\mscr{A}$. The complete recursion is given by Algorithm~\ref{alg:construction}.
\begin{algorithm}
\caption{Construct PC with respect to $W$}\label{alg:construction}
\begin{algorithmic}
\Procedure{ConstructCircuit}{$w$}
    \If{$w$ is a leaf node $X_i$}
        \State \Return $p(X_i \given \mbf{C}_{X_i})$
    \EndIf
    \State $l, r \gets \textsc{Children}(w)$
    \State $\bigoplus_{\mbf{X}_l, \mbf{C}_l}\!\gets\!\textsc{ConstructCircuit}(l)$
    \State $\bigoplus_{\mbf{X}_r, \mbf{C}_r}\!\gets\!\textsc{ConstructCircuit}(r)$
    \State $\bigoplus_{\mbf{X}_w, \mbf{C}_w}\!\gets\!\sum\limits_{(\mbf{C}_l \cup \mbf{C}_r) \setminus \mbf{C_w}}\!\bigoplus_{\mbf{C}_l} \cdot \bigoplus_{\mbf{C}_r} \cdot p(\mbf{C}_l, \mbf{C}_r \given \mbf{C}_w)$
    \State \textbf{return} $\bigoplus_{\mbf{C}_w}$
\EndProcedure
\end{algorithmic}
\end{algorithm}

\begin{theorem}\label{theorem:pcsize}
Let $h$ be the number of hidden states of the original PC $\mscr{A}$ and $n$ the number of random variables. The number of hidden states of the restructured PC is given by $O(h^{M})$ where $M\!=\!{\max}_{w\in W} |\mbf{C}_l \cup \mbf{C}_r|$ and the size of the restructured PC is bounded by $O(nh^{M'})$ where $M'\!=\!\max_{w \in W} |\mbf{C}_l \cup \mbf{C}_{r} \cup \mbf{C}_w| \leq 2M$. We refer to $M^{\prime}$ as the cardinality of the labelling $\mbf{C}_w$.
\end{theorem}
\begin{proof}
Let $\mscr{A}^\prime$ be the restructured circuit respecting $W$. As described in Algorithm~\ref{alg:construction}, for each inner node $w \in W$, we construct two layers of nodes as shown in Figure~\ref{fig:construction}. By construction, the product layer contains all product nodes respecting the vtree node $w$ and its cardinality is given by $O(h^{|\mbf{C}_l \cup \mbf{C}_r|})$; we set $M\!:=\!\max_{w \in W} |\mbf{C}_l \cup \mbf{C}_r|$ and it follows that the hidden states size of $\mscr{B}$ is given by $O(h^M)$. Similarly, the number of edges in the sum layer is given by $O(h^{|\mbf{C}_l \cup \mbf{C}_r \cup \mbf{C}_w|})$ and the number of product edges is given by $O(h^{|\mbf{C}_l \cup \mbf{C}_r|})$; since there are $O(n)$ vtree nodes in total, the total number of edges in $\mscr{B}$ is given by $O(nh^{M^\prime})$, with $M^{\prime} = \max_{w \in W} |\mbf{C}_l \cup \mbf{C}_r \cup \mbf{C}_w|$.
\end{proof}

\begin{remark}
By Theorem~\ref{theorem:pcsize}, the restructured PC $\mscr{A}^\prime$ has hidden state size $O(h^{M})$, which gives a circuit of size $\Theta(nh^{2M})$ only if $\mscr{A}^\prime$ is densely connected. In fact, we will show in Section~\ref{sec:multiplication} and~\ref{sec:depth_reduction} that the restructured PCs are often sparsely connected, resulting in sizes much smaller than $O(nh^{2M})$. Thus, while the graphical model representation is useful for reasoning about conditional independencies, the circuit representation allows us to visualize and exploit the sparsity for efficient inference~\citep{DangNeurIPS22, LiuICML24}.
\end{remark}

\subsection{Computing Vtree Labelling}
The next question that immediately arises is how to compute a valid labelling for $W$ with respect to $G_{\pc}$. One naive solution is to set $\mbf{C}_w$ to be $\mbf{Z}$, the set of all latent variables in $G_{\pc}$. However, this is not desirable as $M^\prime\!=\!\max_{w \in W} |\mbf{C}_l \cup \mbf{C}_r \cup \mbf{C}_w|\!=\!|\mbf{Z}|\!=\!n\!-\!1$, resulting in the restructured circuit having exponential size $O(nh^{n-1})$. Hence we present a greedy approach that computes a labelling while trying to minimize $M^\prime$.

The algorithm proceeds top-down on $W$. For the base case where $w$ is the root, we set $\mbf{C}_w := \varnothing$. For the inductive step, let $l$ and $r$ be the children of $w$ and assume that we have computed $\mbf{C}_w$ as a cover for $\mbf{X}_w$ in $G_{\pc}$: we (1) split $G_{\pc}$ into connected components $\{G_i\}$ via $\mbf{C}_w$; then (2) within each connected component $G_i$, we compute a minimum d-separator $\mbf{C}_i$ that blocks all paths between $\mbf{X}_l \cap G_i$ and $\mbf{X}_r \cap G_i$ by calling the sub-routine \textsc{MinimumSeparator}. We set $\mbf{D}_w := \left({\bigcup}_i \mbf{C}_i\right) \cup \mbf{C}_w$ and observe that $\mbf{D}_w$ covers both $\mbf{X}_l$ and $\mbf{X}_r$ in $G_{\pc}$. To compute $\mbf{C}_l$, similarly for $\mbf{C}_r$, we consider all paths starting from $\mbf{X}_l$ and stopping immediately when reaching some $Z_j \in \mbf{D}_w$, and we let $\mbf{C}_l$ to be the set containing all such $Z_j$s. The pseudo code is shown in Algorithm~\ref{alg:compute_label}.
\begin{algorithm}
\caption{Computing $\mbf{C}_w$ for $w\!\in\!W$}\label{alg:compute_label}
\begin{algorithmic}
\Procedure{ComputeLabel}{$w, \mbf{C}_w$}
    \State $\{G_i\} \gets \textsc{ConnectedComponents}(G_{\pc}, \mbf{C}_w)$
    \State $\mbf{C}_{i} \gets \textsc{MinimumSeparator}(G_i, \mbf{X}_l\!\cap\!G_i, \mbf{X}_r\!\cap\!G_i)$
    \State $\mbf{D}_w \gets \left({\bigcup}_i \mbf{C}_i\right) \cup \mbf{C}_w$
    \State $\mbf{C}_l \gets \{Z_j\!\in\!\mbf{D}_w: \textsc{Paths}(\mbf{X}_l, Z_j) \cap \mbf{D}_w\!=\!\{Z_j\} \}$
    \State $\mbf{C}_r \gets \{Z_j\!\in\!\mbf{D}_w: \textsc{Paths}(\mbf{X}_r, Z_j) \cap \mbf{D}_w\!=\!\{Z_j\} \}$
    \State $\textsc{ComputeLabel}(l, \mbf{C}_l)$
    \State $\textsc{ComputeLabel}(r, \mbf{C}_r)$
\EndProcedure
\end{algorithmic}
\end{algorithm}
Note that the \textsc{MinimumSeparator} procedure called in Algorithm~\ref{alg:compute_label} computes a minimum d-separator that blocks all paths between $\mbf{X}_l$ and $\mbf{X}_r$ in the subgraph $G_i$. Even though polytime algorithms for computing minimum d-separators exist in literature~\citep{tian1998finding}, we derive a linear-time algorithm that is easy to implement for our use case, where $G_i$ is a rooted tree with leaves in $\mbf{X}_l$, $\mbf{X}_r$ and $\mbf{C}_w$. We refer readers to the Appendix for details.
\begin{proposition}
Algorithm~\ref{alg:compute_label} computes a valid labelling with respect to $G_{\mscr{A}}$.
\end{proposition}

\begin{proof}
We prove by a top-down induction on $W$ that the labelling $\mbf{C}_w$ computed by Algorithm~\ref{alg:compute_label} is valid. Assume that $\mbf{C}_w$ covers $\mbf{X}_w$ in $G_{\pc}$, we want to show that $\mbf{C}_l$ and $\mbf{C}_r$ satisfy the properties from Definition~\ref{def:correct_labels}. To prove that $\mbf{C}_l$ covers $\mbf{X}_l$, we consider a path from $X_a\!\in\!\mbf{X}_l$ to $X_b\!\in\!\mbf{X} \setminus \mbf{X}_l$. (1) If $X_a$ and $X_b$ are in the same $G_i$, then the path is blocked by $\mbf{C}_i$. (2) If $X_a$ and $X_b$ are in different $G_i$s, then the path contains some node $Z \in \mbf{C}_w$, and we can choose from the path the first $Z\!\in\!\mbf{C}_w$. Then $Z\!\in\!\mbf{C}_l$ by construction, implying that the path is blocked by $\mbf{C}_l$. Hence we conclude that $\mbf{C}_l$ is a cover for $\mbf{X}_l$, satisfying Property~1. To prove that $\mbf{C}_l$ satisfies Property 2, we argue that because $\mbf{C}_r$ and $\mbf{C}_w$ are both subsets of $\mbf{D}_w$, all paths from $\mbf{X}_l$ to $\mbf{C}_r \cup \mbf{C}_w$ will be blocked by $\mbf{C}_l$ by the way that $\mbf{C}_l$ is constructed. We can show that $\mbf{C}_r$ satisfies Property~1~and~3 by the same argument.
\end{proof}

As an example, we (partially) illustrate the application of Algorithm \ref{alg:compute_label} to the target vtree in Figure \ref{fig:labelling}. For node $a$ and its children $b, c$, the only connected component is $G_1 := G_{\mscr{A}}$. A minimum separator of $\mbf{X}_b = \{X_3, X_6, X_4, X_5\}$ and $\mbf{X}_c = \{X_1, X_2\}$ is then given by $\mbf{C}_1 := \{Z_{1:3}, Z_{2:3}\}$. Then $\mbf{D}_1 = \mbf{C}_1 \cup \mbf{C}_a = \{Z_{1:3}, Z_{2:3}\}$ covers both $\mbf{X}_b, \mbf{X}_c$. Tracing the paths from each $X$ variable to $\mbf{D}_1$ finally gives $\mbf{C}_b = \mbf{C}_c = \{Z_{1:3}, Z_{2:3}\}$.

Though Algorithm~\ref{alg:compute_label} computes a valid labelling while greedily minimizing $|\mbf{C}_l \cup \mbf{C}_r \cup \mbf{C}_w|$, we do not know whether $M^{\prime} = \max_{w \in W} |\mbf{C}_l \cup \mbf{C}_r \cup \mbf{C}_w|$ is globally minimized or not. In addition, we hypothesize that if we can find a minimum vtree labelling, then the size of the PC constructed by Algorithm~\ref{alg:construction} is optimal. We leave it as an open problem to design an algorithm that computes \emph{minimum} labellings and prove the optimality of Algorithm~\ref{alg:construction} given a minimum labelling.

Nonetheless we show that Algorithm~\ref{alg:construction} yields novel polynomial-time algorithms for the tasks of PC multiplication and depth-reduction. Specifically, we show that for important subclasses of PCs, we \emph{can} compute vtree labellings of constant or $O(\log n)$ cardinality. We refer readers to Section~\ref{sec:multiplication} and Section~\ref{sec:depth_reduction} for details.

\subsection{Corollaries}
With our restructuring algorithm in hand, we now examine the restructuring of two other types of circuits: namely, deterministic PCs, and logical circuits. 
\begin{definition}[Determinism]
    A sum node is deterministic if for every value $\mbf{x}$ of $\mbf{X}$, at most one child $c$ returns a non-zero value (i.e. $p_{c}(\mbf{x}) > 0$). A PC is determinstic if all of its sum nodes are deterministic.
\end{definition}

Determinism is crucial for tractability of various inference queries such as computing the most likely state (MAP) \citep{peharz2016latent,conaty2017approximation} or computing the entropy of the PC's distribution \citep{shih2020probabilistic,VergariNeurIPS21}. It is thus of interest to ask whether applying our restructuring algorithm maintains determinism. 
\begin{claim}
    Algorithm~\ref{alg:construction} preserves determinism.
\end{claim}
\begin{proof}    
    If the original circuit is deterministic, then each assignment to the observed variables fully determines the values of all latent variables (and thus the latents being conditioned on for the restructuring). Hence the constructed sum nodes are deterministic.
\end{proof}

Although we have focused on probabilistic circuits up to this point, our restructuring algorithm also applies to logical circuits -  in particular, structured-decomposable negation normal form (SDNNF) circuits\footnote{Many other representations, such as the ordered binary decision diagram (OBDD) and deterministic finite automaton (DFA), can be converted efficiently to (deterministic) SDNNFs \citep{amarilli2024circus}.} \citep{pipatsrisawat2008new}. To see this, we use a simple trick: (1) convert the logical circuit into a probabilistic circuit by replacing $\vee$ with $\oplus$ and $\wedge$ with $\otimes$, and assigning positive weights to $\oplus$ edges; (2) restructure the PC; (3) convert the PC back to a logical circuit by replacing $\oplus$ with $\vee$ and $\otimes$ with $\wedge$, and removing the weights. It is immediate that the logical circuits and the corresponding PCs have the same support throughout the process.

It is also not hard to see that this procedure for logical circuits retains determinism, so, e.g, an ordered binary decision diagram (OBDD) can be efficiently restructured into a deterministic SDNNF with the reverse order while maintaining the ability to perform model counting \citep{darwiche2002knowledge}. 

\begin{figure}[t]
    \centering
    \scalebox{1.0}{\begin{tikzpicture}[semithick]
    \node[] (cs) at (3.5, 0) {\scriptsize \textbf{Contiguous \& Structured}};
    \node[] (ls) at (0, 0) {\scriptsize \textbf{Linear}};
    \node[] (bt) at (2.25, -1.5) {\scriptsize \textbf{Log-Depth \& Contiguous}};

    \node[] (abt) at (-1., -1.5) {\scriptsize \textbf{Any Log-Depth}};

    \node[] (ct) at (5.15, -1.5) {\scriptsize  \textbf{Contiguous}};

    \draw[->, line width=0.7mm, blue] (ls) -- (ct) ;
    \draw[->, line width=0.7mm, blue] (ls) -- (bt) ;
    \draw[->, line width=0.3mm] (ls) -- (abt) ;
    \draw[->, line width=0.7mm, dashed, blue] (cs) -- (bt);
    \draw[->, line width=0.3mm] (cs) -- (abt);

    \draw[line width=0.4mm] (-0.9,0.5)--(-0.4,0.5);
    \draw[line width=0.4mm,densely dashed] (2,0.5)--(2.5,0.5);

    \node[] (poly) at (0.4, 0.5) {\scriptsize Polynomial};
    
    \node[] (poly) at (3.7, 0.5) {\scriptsize Quasi-Polynomial};

    \end{tikzpicture}}
    \caption{Summary of restructuring results; the top/bottom rows indicate the source/target structures.
    {\color{blue}\textbf{Blue and bold}} arrows indicate novel complexity results. For the two other arrows that were known to be polynomial-time, our approach yields more efficient algorithms amenable to practical implementation.
    }
    \label{fig:theory_summary}
\end{figure}

\section{PC MULTIPLICATION}
\label{sec:multiplication}
One important application of restructuring PCs is circuit multiplication: given two PCs $\mscr{A}$ and $\mscr{B}$, the goal is to construct a tractable PC $\mscr{C}$ such that $p_{\mscr{C}}(\mbf{x}) \propto p_{\mscr{A}}(\mbf{x}) \cdot p_{\mscr{B}}(\mbf{x})$. PC multiplication was previously only addressed for structured PCs respecting the \emph{same} vtree~\citep{shen2016tractable, VergariNeurIPS21}. Circuit restructuring immediately gives us a means of multiplying two structured circuits respecting \emph{different} vtrees, as we can simply restructure one of them to be compatible with the other. Though the restructured PC will in general have exponential size, in this section, we consider practical cases where circuit multiplications with respect to \emph{different} vtrees is tractable. Further, we will show how ``on-the-fly'' restructuring can enable circuit multiplication even when one of the circuits is not structured-decomposable.

We start by introducing a new structural property of tractable PCs called \emph{contiguity}.\footnote{ \chg{\citet{amarilli2017enumeration} defined a similar concept called \emph{compatible order} in the context of logical circuits.
}}
\begin{definition}[Contiguity]
Given the canonical ordering of  variables $X_1, X_2, \dots, X_n$, a PC node is \emph{contiguous} if its scope is of the form $X_{a}, X_{a+1}, \dots, X_{b}$ for some $1\!\leq\!a\!\leq\!b\!\leq\!n$. A smooth and decomposable PC is contiguous if all of its nodes are contiguous.
\end{definition}
Contiguous PCs subsume many widely-used subclasses of PCs such as Hidden Markov Models~(HMMs).
Note that a contiguous circuit is not necessarily structured-decomposable and $0.5 \otimes p(X_1) \otimes p(X_2, X_3) \oplus 0.5 \otimes p(X_1, X_2) \otimes p(X_3)$ is such an example. 
The fact that contiguous PCs need not be structured-decomposable also suggests new designs of circuits that can be tractably multiplied. One important example is probabilistic context-free grammars~(PCFGs).
\begin{theorem}
\label{theorem:cfg}
Let $G$ be a PCFG consisting of $m$ production rules. For sequences of length $n$, $G$ can be represented as a contiguous yet non-structured PC $\mscr{A}$ of size $O(mn^3)$; i.e. $p_{\mscr{A}}(\mbf{x})$ computes the probability of $\mbf{x}$ being derived from $G$, for all $\mbf{x}$ of length $n$.
\end{theorem}
The algorithm for constructing the PC representation for a PCFG is in spirit similar to the \emph{CYK algorithm}~\citep{kasami1966efficient,younger1967recognition,cocke1969programming}: given a PCFG, for each contiguous subsegment $[a, b]$ of $1, 2, \dots, n$, and each non-terminal $N$ in the PCFG, we can construct a sum node $u$ (recursively) such that the sub-circuit rooted at $u$ represents the distribution over the set of strings on the segment $[a, b]$ that can be derived from $N$.

In the following, we consider the problem of multiplying two contiguous PCs $\mscr{A}$ and $\mscr{B}$. Intuitively, random variables forming contiguous scopes are more likely to be covered by vtree labellings of small cardinalities, which would give efficient restructuring by Section~\ref{sec:restructuring}. We start by considering the simpler case where both circuits are structured. We then generalize our results by allowing one of them to be non-structured. Figure \ref{fig:theory_summary} summarizes our main results (Theorems \ref{theorem:linear_structured}, \ref{theorem:depth_structured}, \ref{theorem:linear_contiguous}). 

\subsection{Multiplying Contiguous Structured PCs}
\begin{figure}
    \centering
    \includegraphics[width=0.45\linewidth]{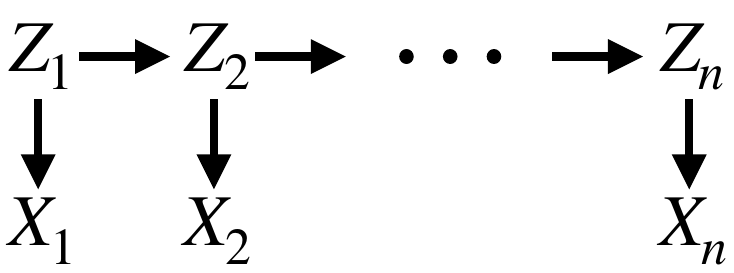}
    \caption{$G_{\mscr{A}}$ for $\mscr{A}$ with a linear vtree}
    \label{fig:hmm}
\end{figure}

\underline{Case 1.} For the multiplication of contiguous PCs $\mscr{A}$ and $\mscr{B}$, we start by considering the case when $\mscr{A}$ is a contiguous structured PC respecting the \emph{linear vtree} $V$ and $\mscr{B}$ is a contiguous structured PC respecting an arbitrary vtree $W$. It follows from Section~\ref{subsec:pc_to_graphical} that the Bayesian network representation for $\mscr{A}$ is a hidden Markov model~\citep{rabiner1989tutorial}, as shown in Figure~\ref{fig:hmm}. By the definition of contiguity, each node $w\!\in\!W$ has a scope of the form $\mbf{X}_{a:b} := \{X_a, \dots, X_b\}$ and we can label it with $\mbf{C}_{a:b} := \{Z_a, Z_{b+1}\}$; in particular, we drop $Z_a$ if $a = 1$ and drop $Z_{b+1}$ if $b = n$.
\begin{restatable}{claim}{clmContiguousLabel}
$\mbf{C}_{a:b}$ is a valid vtree labelling of $W$ respecting $G_{\mscr{A}}$ with cardinality $M^{\prime} = 3$.
\end{restatable}
Then it follows from Theorem~\ref{theorem:pcsize} that the size of $\mscr{A}^{\prime}$, i.e., the PC obtained by restructuring $\mscr{A}$ respecting $W$, is bounded by $O(nh^3)$, with $O(|\mscr{A}|^2)$ being a looser bound. Eventually we can compute the product of $\mscr{A}^\prime$ of $\mscr{B}$ tractably by the existing algorithm for multiplying two circuits respecting the same vtree~\citep{shen2016tractable, VergariNeurIPS21}.
\begin{restatable}{theorem}{thmContiguousLinear}
\label{theorem:linear_structured}
    Let $\mscr{A}$ and $\mscr{B}$ be contiguous structured PCs. If $\mscr{A}$ has a linear vtree, then $\mscr{A}$ and $\mscr{B}$ can be multiplied in polynomial time and the size of the product PC is bounded by $O(|\mscr{A}|^2|\mscr{B}|)$.
\end{restatable}

\underline{Case 2.} Then we consider the more general case where $\mscr{A}$ is a contiguous structured PC of depth $d$ respecting vtree $V$ and $\mscr{B}$ is a contiguous structured PC with an arbitrary vtree $W$. Similarly to the previous case, our goal is to come up with a small labelling of $W$ with respect to $G_{\mscr{A}}$. Since $\mscr{A}$ is contiguous, its vtree $V$ can be viewed as a \emph{segment tree}~\citep{cormen2022introduction}. Algorithm~\ref{alg:segment_cover}, which is adapted from the segment tree querying algorithm, computes a cover $\mbf{C}_{a:b} \subseteq G_{\mscr{A}}$ for each contiguous segment $\mbf{X}_{a:b}$. For each $w\!\in\!W$, $\mbf{X}_w = X_{a:b}$ for some $1 \leq a \leq b \leq n$ and we set $\mbf{C}_w = \mbf{C}_{a:b} =  \textsc{SegmentCover}(V, X_{a:b})$. 
\begin{algorithm}
\caption{Compute Cover for Segment $\mbf{X}_{a:b}$}\label{alg:segment_cover}
\begin{algorithmic}
\Procedure{SegmentCover}{$v$, $\mbf{X}_{a:b}$}
    \If{$\mbf{X}_{a:b} = \varnothing$}
        \State \Return $\varnothing$
    \EndIf
    \If{$\mbf{X}_{a:b} = \mbf{X}_v$}
        \State \Return $\{Z_v\}$
    \EndIf
    \State $l, r \gets \textsc{Children}(v)$
    \State $\mbf{L} \gets \textsc{SegmentCover}(l, \mbf{X}_l \cap \mbf{X}_{a:b})$
    \State $\mbf{R} \gets \textsc{SegmentCover}(r, \mbf{X}_r \cap \mbf{X}_{a:b})$
    \State \textbf{return} $\mbf{L} \cup \mbf{R}$
\EndProcedure
\end{algorithmic}
\end{algorithm}

\begin{restatable}{proposition}{propContiguousLabelingDepth}
$\mbf{C}_w$ is a valid vtree labelling with respect to $G_{\mscr{A}}$, and $|\mbf{C}_w|\!\leq\!4d$ for all $w$.
\end{restatable}
In addition to the fact that $\mbf{C}_w$ is a valid labelling, $|\mbf{C}_w|\!\leq\!4d$ follows from the fact that the number of nodes visited at each level of $V$ is at most $4$ (see Appendix); hence the cardinality of $\mbf{C}_w$ is bounded by $12d$. Then following a similar analysis, we have the following results for multiplying two contiguous PCs.
\begin{restatable}{theorem}{thmContiguousDepth}
\label{theorem:depth_structured}
    Let $\mscr{A}$ and $\mscr{B}$ be contiguous structured PCs. Let $d$ be the depth of the vtree for $\mscr{A}$, then $\mscr{A}$ and $\mscr{B}$ can be multiplied in time $O(|\mscr{A}|^{12d}|\mscr{B}|)$.
\end{restatable}
\begin{restatable}{corollary}{corQuasi}
If $\mscr{A}$ is of depth $O(\log n)$ then $\mscr{A}$ and $\mscr{B}$ can be multiplied in quasi-polynomial time.
\end{restatable}
\begin{remark}
In this work, we assumed product nodes always have two children and binary vtrees. Hence the depths of PCs are lower-bounded by $\Omega(\log n)$ under such assumptions. However, if we allow product nodes to have arbitrarily many children, we can have PCs of smaller or even constant depths~\citep{raz2009lower} and we hypothesize that Theorem~\ref{theorem:pcsize} can be adapted to such generalized cases thus giving a polynomial-time algorithm for multiplying contiguous structured circuits of constant depths.
\end{remark}
\subsection{Generalization to Non-structured PCs}
Thus far, we have assumed that both $\mscr{A}$ and $\mscr{B}$ are structured PCs. We claim that we can further generalize our results by removing the constraint that $\mscr{B}$ has to be structured, and Theorems~\ref{theorem:linear_structured} and~\ref{theorem:depth_structured} would still hold.
We illustrate the idea by showing how to multiply a contiguous structured PC $\mscr{A}$ respecting a linear vtree and an arbitrary contiguous PC $\mscr{B}$. Since $\mscr{B}$ is not structured decomposable, we cannot restructure $\mscr{A}$ to the vtree of $\mscr{B}$. However, we can use the same idea as Algorithm~\ref{alg:construction} to restructure $\mscr{A}$ ``on-the-fly'' as we multiply it with $\mscr{B}$ in a bottom-up way. Specifically, for each possible scope $\mbf{X}_{a:b}$ that appears in $\mscr{B}$, we recursively construct circuit representations for the functions $p_{q}(\mbf{X}_{a:b})\cdot p_{\mscr{A}}(\mbf{X}_{a:b} \given Z_a\!=\!i, Z_b\!=\!j)$ for $i, j$ and $q\!\in\!\mscr{B}$ with scope $\mbf{X}_{a:b}$. The recurrence relation is similar to that of Algorithm~\ref{alg:construction} and we refer readers to the Appendix for details.
\begin{restatable}{theorem}{thmContiguousNonstructured}
\label{theorem:linear_contiguous}
    Let $\mscr{A}$ and $\mscr{B}$ be contiguous PCs with $\mscr{B}$ not necessarily structured. If $\mscr{A}$ is structured respecting the linear vtree, then $\mscr{A}$ and $\mscr{B}$ can be multiplied in polynomial time and the size of the product PC is bounded by $O(|\mscr{A}|^2|\mscr{B}|)$.
\end{restatable}

As an explicit application of circuit multiplication, let us consider constrained text generation, in which linear PCs (HMMs) are multiplied with deterministic finite automata (DFAs) representing the logical constraint~\citep{zhang2024adaptable}. Our results imply that one can also multiply HMMs with contiguous logical circuits such as a sentential decision diagrams~(SDDs)~\citep{darwiche2011sdd}, which have been shown to be exponentially more expressive efficient~\citep{bova2016sdds}.
It also follows from Theorem~\ref{theorem:cfg} that we can tractably multiply HMMs with (P)CFGs, subsuming a recent result that HMMs can be tractably marginalized over unambiguous CFGs~\citep{marzouk2022marginal}. Our results on multiplying contiguous PCs open up new possibilities for generalizing~\citep{zhang2024adaptable} to much broader families of circuits and constraints.

\section{PC DEPTH REDUCTION}
\label{sec:depth_reduction}

\begin{algorithm}
\caption{Depth Reduction Vtree}\label{alg:depth_reduction_vtree}
\begin{algorithmic}[1]
\Procedure{BalancedVtree}{$\vtree, \mbf{S}_{l} = \emptyset, \mbf{S}_{r} = \emptyset$}
    \If{$|\vtree| = 1$}
        \State \textbf{return} \textsc{leaf}($\vtree$; $\mbf{S}_l \cup \mbf{S}_r)$
    \EndIf
    \State $v \gets \textsc{root}(\vtree)$
    \State $l, r \gets \textsc{Children}(v)$
    \While{$|V_r| > \frac{2}{3}|V|$} \label{algline:dr_while_start}  
        \State $v \gets r$
        \State $l, r \gets \textsc{Children}(v)$ \Comment{assume $|\vtree_l|\!\leq\!|\vtree_r|$}
        
    \EndWhile \label{algline:dr_while_end}
    \State $\vtree_{l}' \gets \textsc{BalancedVtree}(\vtree_{[v \mapsto l]}, \mbf{S}_l, \{Z_{v}\})$
    \State $\vtree_{r}' \gets \textsc{BalancedVtree}(\vtree_{r}, \{Z_{v}\}, \mbf{S}_r)$
    \State \textbf{return} $\textsc{join}(\vtree_{l}', \vtree_{r}'; \mbf{S}_l \cup \mbf{S}_r)$ 
\EndProcedure
\end{algorithmic}
\end{algorithm}

Depth reduction of a probabilistic circuit refers to the construction of an equivalent circuit with reduced depth, e.g. to a depth logarithmic in the number of variables. A depth reduction algorithm for general circuits is known~\citep{valiant1983fast, raz2008balancing, yin2024expressive} but does not take advantage of structuredness. We show how to reduce a structured-decomposable circuit to an equivalent log-depth circuit by restructuring. Unlike in multiplication, we are not concerned with restructuring to a particular vtree, but rather \emph{any} log-depth vtree. The key step is thus to identify a log-depth vtree such that restructuring to that vtree using Algorithm~\ref{alg:construction} (and some valid choice of labels) results in at most a polynomial increase in size.

Algorithm \ref{alg:depth_reduction_vtree} constructs a log-depth vtree labelling of constant cardinality. Intuitively, each step of the algorithm breaks a vtree down into two connected components, which are then depth-reduced recursively. One selects a single vtree node by traversing the vtree top-down, until the split would be balanced in the sense that the two connected components have size between $\frac{1}{3}$ and $\frac{2}{3}$ of the input vtree (Lines \ref{algline:dr_while_start}-\ref{algline:dr_while_end}). The algorithm simulataneously constructs a valid label for the vtree node. The \textsc{join} routine then returns a labelled vtree that consists of a single root node with the aforementioned label, connected to the depth-reduced vtrees for the components. Note that the algorithm produces exactly one vtree node for each vtree node in the original vtree; we can thus write $v(w)$ for the node in $V$ corresponding to $w$. Then we have the following result:
\begin{theorem}[Depth Reduction] \label{thm:depth_reduction}
    Given any vtree $\vtree$, Algorithm \ref{alg:depth_reduction_vtree} returns a vtree $W$ of depth $O(\log|\vtree|)$ with a valid labelling of cardinality \chg{$M'$ at most} $3$.    
\end{theorem}
\begin{proof}
     The depth reduction to $O(\log|\vtree|)$ is achieved as the algorithm increases the depth by one in each recursive call, but reduces the vtree size by a multiplicative factor. The validity condition holds due to the separation into connected components (the labels can also be obtained from Algorithm \ref{alg:compute_label}). The value of $M'$ follows by noting that $\mbf{S}_l$ and $\mbf{S}_r$ are either empty or singleton sets, and that the algorithm produces $\mbf{C}_w\!=\!\mbf{S}_l \cup \mbf{S}_r$, $\mbf{C}_l\!=\!\mbf{S}_l \cup \{Z_{v(w)}\}$, and $\mbf{C}_r\!=\!\{Z_{v(w)}\} \cup \mbf{S}_l$ where $Z_{v(w)}$ for each inner node $w\!\in\!W$.
\end{proof}
\begin{remark}
Firstly, the depth-reduced PC retains structuredness, which is not guaranteed by the existing depth-reduction algorithms. Secondly, exploiting structuredness and tracking the hidden state size enables a more fine-grained analysis of the size of the depth-reduced circuit. Since the size of the original circuit is $O(nh^2)$, using the known cubic bound on the size of the depth-reduced circuit \citep{raz2008balancing} gives $O(n^3h^6)$. However, by Theorem \ref{thm:depth_reduction}, we see that $M^\prime\!=\!\max_{w \in W} |\mbf{C}_l, \mbf{C}_r, \mbf{C}_W|\!\leq\!3$ and so by Theorem \ref{theorem:pcsize} we immediately obtain a much tighter bound of $O(nh^3)$ for the size of the resulting circuit.
\end{remark}
\begin{restatable}{corollary}{corDepthReduction}
    Any structured PC over $n$ variables and with hidden state size $h$ can be restructured to a structured PC of depth $O(\log n)$ and size $O(nh^3)$ that represents the same distribution.
\end{restatable}
While this result is of independent theoretical interest, the sub-quadratic complexity of $O(nh^3)$ also opens up practical applications of depth-reduction. Almost all PC inference and learning algorithms involve forward/backward passes through the computation graph, where computation is only parallelized across nodes of the same depth such that $O(\text{depth of PC})$ sequential computations are required. This is problematic when the number of variables $n$ is large, as is often the case in applications such as computational biology \citep{DangRECOMB22}. In such cases, depth reduction can be a practical strategy where the improved parallelism outweighs the increased circuit size. 

\section{RELATED WORK}
Probabilistic circuits have emerged as a unifying representation of tractable probabilistic models \citep{ProbCirc20,sidheekh2024building}, such as sum-product networks \citep{poon2011sum}, cutset networks \citep{rahman2014cutset}, probabilistic sentential decision diagrams \citep{kisa2014probabilistic} and probabilistic generating circuits~\citep{ZhangICML21, harviainen2023inference, agarwal2024probabilistic, BroadrickUAI24}. Significant effort has been devoted to learning PC structures to fit data \citep{liang2017learning,dang2020strudel,yang2023bayesian}, but the implications for the structure-dependent queries have been less studied. 
We bridge this gap by providing a general restructuring algorithm with specific cases of (quasi-)polynomial complexity.

As tractable representations of distributions, PCs have been employed extensively as a compilation target for inference in graphical models \citep{darwiche2003differential,chavira2008probabilistic,rooshenas2014learning}. Hidden tree-structured Bayesian networks have also been used as a starting point for the learning of a probabilistic circuit \citep{dang2020strudel,LiuNeurIPS21,DangNeurIPS22}. A particularly useful analysis technique for learning probabilistic circuits has been to interpret them as latent variable models \citep{peharz2016latent}.  Decomposable and smooth PCs can be interpreted as Bayesian networks by introducing a latent variable for each sum node in the PC \citep{zhao2015relationship}. Our conversion from structured PC to Bayesian network is most closely related to the \emph{decompilation} methods of \citet{butz2020sum,papantonis2023transparency}, but we do not assume the PC has been compiled from a Bayesian network.

The seminal work of \citet{valiant1983fast} showed that any poly-size arithmetic circuit can be transformed into an equivalent circuit of polylogarithmic depth. \citet{raz2008balancing} show that this procedure maintains syntactic multilinearity (decomposability). Recently, \citet{yin2024expressive} showed a quasipolynomial upper bound on converting decomposable and smooth PCs to tree-shaped PCs via a depth-reduction procedure. Our application of restructuring focuses on structured-decomposable circuits and shows a tighter bound based on a graphical model interpretation.

\section{DISCUSSION}
We introduce the problem of \emph{restructuring} probabilistic circuits, and develop a general algorithm for restructuring a structured-decomposable circuit to any target vtree structure.  Our method exploits an interpretation of structured-decomposable circuits as latent tree Bayesian networks, which enables recursive construction of a circuit respecting the target vtree using probabilistic semantics of the Bayesian network. As concrete applications of restructuring, we show how to tractably multiply two circuits which do not necessarily share the same structure but satisfy a \emph{contiguity} property, and show how to restructure a circuit to log-depth with a sub-quadratic increase in size. 

Our work takes a step towards understanding when and how the space of structured circuits -{}- hierarchical tensor factorizations \citep{loconte2024relationship} -{}- can be ``connected'' through tractable transformations. Many interesting theoretical problems remain open: for instance, the optimality of our algorithm, and circuit lower bounds for particular structures. We also envisage that our restructuring algorithm will enable new methodology both within the PC community and broadly in structured representations of high-dimensional tensors and probability distributions.

\subsection*{Acknowledgements}
This work was funded in part by the DARPA ANSR program under award FA8750-23-2-0004, the DARPA CODORD program under award HR00112590089, NSF grant \#IIS-1943641, and gifts from Adobe Research, Cisco Research, and Amazon. This work was done in part while BW, MA and GVdB were visiting the Simons Institute for the Theory of Computing.

\bibliography{sample_paper}

\appendix
\onecolumn
\aistatstitle{Supplementary Materials}
\unskip
\section{ADDITIONAL PROOFS}
\propAug*
\begin{proof}
    Suppose that $\pc$ is a structured decomposable and smooth PC respecting vtree $\vtree$. Write $\prods(v)$  and $ \text{sum}(v)$ for the set of product and sum nodes with scope $\mbf{X}_v$. The augmented PC $\pc_{\text{aug}}$ is decomposable as if leaves with scope $\{Z_{v}\}$ were contained in (the subcircuits rooted at) two different children $t_1, t_2$ of a product node, then their parents (nodes in $\prods(v)$) would be contained in $t_1, t_2$, which is a contradiction of decomposability of $\pc$. It is also smooth as for any sum node, if one sum node contains some leaf with scope $\{Z_v\}$, then it contains some node in $\text{sum}(v)$, hence by smoothness of $\pc$ all sum nodes contain some node in $\text{sum}(v)$ and thus some leaf with scope $\{Z_v\}$.

    Consider the standard marginalization algorithm for PCs \citep{darwiche2003differential,ProbCirc20}, where one replaces each leaf whose scope is contained within the variables being marginalized out with the constant $1$. This correctly marginalizes the function represented by the PC if the PC is decomposable and smooth. If we marginalize over all newly introduced latents $\mbf{Z}$, it is immediate that the resulting PC represents the same function as $\pc$.
\end{proof}

\thmPCBN*
\begin{proof}
    In Section \ref{subsec:pc_to_graphical} we described a Bayesian network $p_G = p^*$ with the required graph. It remains to show that this network represents the same distribution as $\pc$. We will do this by showing that the Bayesian network has the same distribution as the augmented PC, i.e. $p_G(\mbf{X}, \mbf{Z}) = p_{\pc_{\text{aug}}}(\mbf{X}, \bm{Z})$.

    The key observation is to consider the \emph{induced trees} of the augmented PC~\citep{zhao2016unified}:

\begin{definition}[Induced Trees] 
    Given a decomposable and smooth circuit $\pc$, let $T$ be a subgraph of $\pc$. We say that $T$ is an induced tree of $\pc$ if (1) $\textsc{root}(\pc) \in T$; (2) If $t \in T$ is a sum node, then exactly one child of $t$ (and the corresponding edge) is in $T$; and (3) If $t \in T$ is a product node, then all children of $t$ (and the corresponding edges) are in $T$.
\end{definition}

It is easy to see that an induced tree $T$ is indeed a tree, as otherwise decomposability would be violated. Let $\mathcal{T}$ be the set of all induced trees of $\pc_{\text{aug}}$. Each induced tree defines a function:
\begin{equation} \label{eqn:tree_value}
    p_{\pc_{\text{aug}}, T}(\mbf{X}, \mbf{Z}) := \prod_{(t_i, t_j) \in \textsc{sumedges}(T)} \pcweight_{t_i, t_j} \prod_{t \in \textsc{leaves}_{\mbf{X}}(T)} \pcleaffn_{t}(X_{\scope(t)}) \prod_{t \in \textsc{leaves}_{\mbf{Z}}(T)} \pcleaffn_{t}(Z_{\scope(t)})
\end{equation}
where $\textsc{sumedges}(T)$ denotes the set of outgoing edges from sum nodes in $T$, and $\textsc{leaves}_{\mbf{X}}(T), \textsc{leaves}_{\mbf{Z}}(T)$ denote the set of leaf nodes in $T$ with scope corresponding to a variable in $\mbf{X}, \mbf{Z}$ respectively. The distribution of the augmented PC is then in fact given by the sum of these functions over all induced trees:

\begin{proposition}[\citet{zhao2016unified}] \label{prop:zhao}  $p_{\pc_{\text{aug}}}(\mbf{X}, \mbf{Z}) = \sum_{T \in \mathcal{T}} p_{\pc_{\text{aug}}, T}(\bm{X}, \bm{Z})$.
\end{proposition}

Now, let $\text{path}(v, i, j)$ be a predicate indicating whether there is a path between $t_{p, j}$ and $t_{v, i}$ (where we use $p$ to denote the parent vtree node of $v$, and as before e.g. $t_{v, i}$ indicates the product node with scope $\mbf{X}_v$ and corresponding to $Z_v = i$). We will consider two cases depending on the value of the latents. Specifically, we will say that an assignment $\bm{z}$ is \emph{consistent} if $\text{path}(v, z_v, z_p)$ holds for all non-root inner nodes in the vtree, and \emph{inconsistent} otherwise. 

If an assignment $\bm{z}$ is inconsistent, then for any assignment $\mbf{x}$ of the observed variables, we have that $p_G(\mbf{x}, \mbf{z}) = 0$ by definition of the Bayesian network distribution. Now consider any induced subtree $T \in \mathcal{T}$. Each $T$ must contain one product node for every variable scope. In particular, $T$ must contain some product node $t_{v, i}$ such that $z_v \neq i$ (otherwise, since $\mbf{z}$ is inconsistent, a (connected) tree would be impossible). We then have $p_{\pc_{\text{aug}}, T}(\bm{x}, \bm{z}) = 0$ for all $\bm{x}$, as Equation \ref{eqn:tree_value} then contains a leaf function $f_t(z_{v}) = \mathds{1}_{z_{v} = i} = 0$. Thus $p_{\pc_{\text{aug}}}(\mbf{x}, \mbf{z}) = \sum_{T \in \mathcal{T}} p_{\pc_{\text{aug}}, T}(\bm{x}, \bm{z}) = 0$ for any $\mbf{x}$.

If an assignment $\mbf{z}$ is consistent, note that, by our assumption of alternating sums and products,  there can be exactly one path from $t_{p, j}$ to $t_{v, j}$, as $t_{p, j}$ has a unique sum node child with scope containing $\mbf{X}_v$, and this sum node must immediately have $t_{v, j}$ as a child. Thus there is exactly one induced tree $T$ containing $t_{v, z_v}$ for all (non-root) inner vtree nodes $v$. Further, examining the definition of the Bayesian network distribution $p_G(\bm{X}, \bm{z})$, this exactly matches the definition of $p_{\pc_{\text{aug}}, T}(\bm{X}, \bm{z})$: each sum node edge weight in the tree corresponds to a sum node edge weight along a path from some $t_{p, z_p}$ to $t_{v, z_v}$ and thus the CPT of $Z_v$ given $Z_p$ (the root sum node edge weight corresponds to the CPT for $Z_{\textsc{root}(V)}$), and each leaf node distribution for observed variables corresponds to the CPT for that variable given its parent.

Thus we have shown that $p_{\pc_{\text{aug}}}(\bm{X}, \bm{Z}) = p_G(\bm{X}, \bm{Z})$, as required.
\end{proof}
\propContiguousLabelingDepth*
\begin{proof}
We first show that $\mbf{C}_w$ satisfies the following properties:
\begin{enumerate}[noitemsep, leftmargin=*]
    \item $X_{a:b} = \bigcup_{Z_i \in \mbf{C}_{a:b}} \textsc{Leaves}(Z_i)$ is a \emph{disjoint} union.
    \item If $a\!\leq\!c\!\leq d\!\leq b$, then for $Z\!\in\!\mbf{C}_{c:d}$, there exists $Z^\prime \in \mbf{C}_{a:b}$ such that $Z^\prime$ is an ancestor of $Z$ in $G_{\mscr{A}}$.
\end{enumerate}
Property~1 follows from the proof of correctness of the segment tree querying algorithm. Property~2 follows from Property~1 together with the key observation that we can compute $\mbf{C}_{c:d}$ via $\bigcup_{Z_i \in \mbf{C}_{a:b}} \textsc{SegmentCover}(Z_i, \mbf{X}_{c:d} \cap \textsc{Leaves}(Z_i))$. Let $w = a\!:\!b$ be a node in $W$ with children $l = a\!:\!c$ and $r = c+1\!:\!b$; it follows from Property 1 that $\mbf{C}_{a:b}$ covers $\mbf{X}_{a:b}$ and $\mbf{C}_{a:c}$ blocks all paths from $\mbf{X}_{a:c}$ to $\mbf{C}_{c+1:b}$; it follows from Property 2 that $\mbf{C}_{a:c}$ blocks all paths from $\mbf{X}_{a:c}$ to $\mbf{X}_{a:b}$. Hence we conclude that $\mbf{C}_w$ is a valid labelling. A minor catch is that $\mbf{C}_w$ may contain variables in $\mbf{X}$, but we can replace them by their parent in $G_{\mscr{A}}$ without affecting the validity of $\mbf{C}_w$.

Now we show that $|\mbf{C}_w| \leq 4d$ and it suffices to show that for each layer of the vtree $V$, the number of nodes $v$ visited by Algorithm~3 is at most $4$. We prove this by a top-down induction. For the base case, the root node is the only node in the first layer hence the number of nodes visited is 1. Now assume that for the $k$-th layer, the number of nodes visited is at most $4$, and denote them by $v_1$, $v_2$, $v_3$, $v_4$, ordered from ``left'' to ``right''. Note that $\{v_i\}_{1 \leq i \leq 4}$ must form a contiguous segment with $v_2$ and $v_3$ completely covered by $\mbf{X}_{a:b}$; hence the function call on $v_2$ and $v_3$ would return without further recursive calls. For the function calls on $v_1$ and $v_2$, suppose both of them trigger recursive calls, in the worse case scenario, each would trigger at most two recursive calls; hence, the number of nodes visited in the $(k+1)$th layer is at most $4$. Since there are $d$ layers in total, the total number of nodes returned is at most $4d$.
\end{proof}
\corDepthReduction*
\begin{proof}
    By Theorem \ref{thm:depth_reduction}, given a structured PC $\mbf{X}$ over $n$ variables with hidden state size $h$ and respecting vtree $V$, we can generate a vtree $W$ of depth $O(\log n)$ and with labelling cardinality $M' = 3$. Thus, by Theorem \ref{theorem:pcsize} we can construct a PC representing the same function and respecting vtree $W$ of size $O(nh^3)$. The depth of the PC is then also $O(\log n)$ as we have assumed alternating sum and product nodes, so the depth of the circuit is at most double that of the vtree.
\end{proof}

\section{Computing Minimum D-separators}
\begin{algorithm}
\caption{Computing minimum d-separators for tree-shaped Bayesian network $G$ rooted at $Z$}
\label{alg:minimum_seperator}
\begin{algorithmic}
\Procedure{MinimumSeparator}{$Z$, $\mbf{A}$, $\mbf{B}$}
    \If{$\mbf{A}\!=\!\varnothing$ and $\mbf{B}\!=\!\varnothing$}
        \State \Return $\varnothing, \varnothing, \varnothing$
    \EndIf    
    \If{$\mbf{B}\!=\!\varnothing$}
        \State \Return $\{Z\}, \varnothing, \varnothing$
    \EndIf
    \If{$\mbf{A}\!=\!\varnothing$}
        \State \Return $\varnothing, \{Z\}, \varnothing$
    \EndIf
    \For{$Z_i\!\in\!\textsc{Children}(Z)$}
        \State $\mbf{C}_{i, \mbf{A}}, \mbf{C}_{i, \mbf{B}}, \mbf{C}_{i} \gets \textsc{MinimumSeparator}($
        \State \quad \quad $Z_i, \mbf{A}\!\cap\!\textsc{Leaves}(Z_i), \mbf{B}\!\cap\!\textsc{Leaves}(Z_i))$
    \EndFor
    \State $\mbf{C}_{\mbf{A}} \gets \textsc{Min}({\bigcup}_i \mbf{C}_i \cup \{Z\}, {\bigcup}_i \mbf{C}_{i,\mbf{A}})$
    \State $\mbf{C}_{\mbf{B}} \gets \textsc{Min}({\bigcup}_i \mbf{C}_i \cup \{Z\}, {\bigcup}_i \mbf{C}_{i,\mbf{B}})$    
    \State $\mbf{C} \gets \textsc{Min}(\mbf{C_A}, \mbf{C_B})$
    \State \textbf{return} $\mbf{C_A}, \mbf{C_B}, \mbf{C}$
\EndProcedure
\end{algorithmic}
\end{algorithm}
Let $G$ be a tree-shaped Bayesian network rooted at $Z$; in particular, assume that the leaves of $G$ $\subseteq \mbf{X}\!\cup\!\mbf{Z}$ and the internal nodes of $G$ $\subseteq \mbf{Z}$ (e.g. Figure~\ref{fig:bn}). Then, we want to prove that Algorithm~\ref{alg:minimum_seperator} computes a minimum d-separator $\mbf{C} \subseteq G$ for $\mbf{A}, \mbf{B} \subseteq \mbf{X}$.

As shown in Algorithm~\ref{alg:minimum_seperator}, given tree-shaped Bayesian network $G$ rooted at $Z$, the procedure $\textsc{MinimumSepartor}$ computes three sets of latent variables $\mbf{C_A}$, $\mbf{C_B}$ and $\mbf{C}$. Specifically, we shall prove the following properties: 
\begin{enumerate}[noitemsep,leftmargin=*]
    \item $\mbf{\mbf{C_A}}$ is a minimum d-separator between $\mbf{A}$ and $\mbf{B}$ that also blocks all paths from $\mbf{A}$ to $Z$.
    \item $\mbf{\mbf{C_B}}$ is a minimum d-separator between $\mbf{A}$ and $\mbf{B}$ that also blocks all paths from $\mbf{B}$ to $Z$.
    \item Either $\mbf{C_A}$ or $\mbf{C_B}$ is a minimum d-separator between $\mbf{A}$ and $\mbf{B}$ in $G$ (rooted at $Z$); hence $\mbf{C}$ is a minimum d-separator between $\mbf{A}$ and $\mbf{B}$ in $G$.
\end{enumerate}
\begin{proof}[Proof of Property 3.]
It suffices to show that $P :=$ \{d-separators between $\mbf{A}$ and $\mbf{B}$\} and $Q :=$ \{d-separators between $\mbf{A}$ and $\mbf{B}$ that blocks all paths from $\mbf{A}$ to $Z$\} $\cup$ \{d-separators between $\mbf{A}$ and $\mbf{B}$ that blocks all paths from $\mbf{B}$ to Z\} are the same set. It is obvious that $Q \subseteq P$ and let's prove that $P \subseteq Q$.
Let $\mbf{C}$ be a d-separator between $\mbf{A}$ and $\mbf{B}$ in $G$, then $\mbf{C}$ either blocks all path from $\mbf{A}$ to $Z$ or blocks all path from $\mbf{B}$ to $Z$; suppose not, then there is a path connecting $\mbf{A}$ and $\mbf{B}$ through Z; contradiction.
\end{proof}

\begin{proof}[Proof of Property 1 and Property 2.] 
We prove Property 1 (and Property 2) by a bottom-up induction on $G$. First of all it is easy to verify that the three base cases, i.e. $\mbf{A} = \varnothing$ and $\mbf{B} = \varnothing$,
$\mbf{A} = \varnothing$ and $\mbf{B} = \varnothing$, are correct.

We now prove Property 1 (the proof for Property 2 is symmetric) by induction; first of all it is clear that both ${\bigcup}_i \mbf{C}_i \cup \{Z\}$ and ${\bigcup}_i \mbf{C}_{i,\mbf{A}}$ form d-separators between $\mbf{A}$ and $\mbf{B}$, and it remains to show that the minimum of these two is a \emph{minimum} d-separator between $\mbf{A}$ and $\mbf{B}$ that blocks all paths from $\mbf{A}$ to $Z$. Suppose, towards a contradiction, let $\mbf{C_A}^\prime$ be such a d-separator of size $< \textsc{Min}(|{\bigcup}_i \mbf{C}_i \cup \{Z\}|, |{\bigcup}_i \mbf{C}_{i,\mbf{A}}|)$. There are two cases:
\begin{itemize}
    \item \textbf{If $\mbf{C_A}^\prime$ contains $Z$}: let $G_i$ be the subtree rooted at $Z_i$ and set ${\mbf{C}_{i}}^\prime = \mbf{C_A}^\prime \cap G_i$. It is immediate that $\mbf{C_i}^\prime$ is a d-separator between $\mbf{A}$ and $\mbf{B}$ in $G_i$ and by assumption, there exists at least one one $i$ such that $|{\mbf{C}_{i}}^\prime| < |\mbf{C}_{i}|$; contradicting the induction hypothesis.
    \item \textbf{If $\mbf{C_A}^\prime$ does not contain $Z$}: let $G_i$ be the subtree rooted at $Z_i$ and set ${\mbf{C}_{i, \mbf{A}}}^\prime = \mbf{C_A}^\prime \cap G_i$. Similarly, it is not hard to see that ${\mbf{C}_{i, \mbf{A}}}^\prime$ is a d-separator between $\mbf{A}$ and $\mbf{B}$ in $G_i$ that blocks all paths from $\mbf{A} \cap G_i$ to $Z_i$. By assumption, there exists at least one $i$ such that $|{\mbf{C}_{i, \mbf{A}}}^\prime| < |{\mbf{C}_{i, \mbf{A}}}|$; contradicting the induction hypothesis.
\end{itemize}
\end{proof}

\section{MULTIPLICATION WITH NON-STRUCTURED CIRCUITS}
Given a contiguous structured PC $\mscr{A}$ respecting a linear vtree and an arbitrary contiguous PC $\mscr{B}$, which is not necessarily structured, we show a recursive algorithm that multiplies $\mscr{A}$ and $\mscr{B}$ in polynomial time. Specifically, for each possible scope $\mbf{X}_{a:b}$ that appears in $\mscr{B}$, we recursively construct circuit representations for the functions $p_{q}(\mbf{X}_{a:b})\cdot p_{\mscr{A}}(\mbf{X}_{a:b} \given Z_a\!=\!i, Z_b\!=\!j)$ for $\oplus$ nodes $q\!\in\!\mscr{B}$ with scope $\mbf{X}_{a:b}$ and $i, j$ hidden states of $\mscr{A}$. In particular, from $p_{\mscr{A}}(\mbf{X}_{a:b} \given Z_a\!=\!i, Z_b\!=\!j)$ we drop $Z_a = i$ if $a = 1$ and drop $Z_{b+1} = j$ if $b = n$, thus $p_{\mscr{A}}(\mbf{X}_{a:b} \given Z_a\!=\!i, Z_b\!=\!j)$ corresponds to $p_{\mscr{A}}(\mbf{X}_{a:b} \given \mbf{C}_{a:b})$ as defined in Case 1. of Section~\ref{sec:multiplication}.

The recurrence relation is similar to that of Algorithm~\ref{alg:construction}. In the following derivation, we use the notations: (1) denote the children of $\oplus$ node $q$ by $c\!\in\!\textsc{Ch}(q)$; (2) denote the children of $\otimes$ node $c$ by $c_1$ and $c_2$; (3) denote the weight of the edge connecting $q$ and $c$ by $w_{qc}$; (4) for each $\otimes$ node $c\!\in\!\textsc{Ch}(q)$, $c$ splits $\mbf{X}_q\!=\!\{X_a, \dots, X_b\}$ into two contiguous segments, and denote them by $\mbf{X}_{c_1}\!=\!\{X_a, \dots, X_{m_c}\}$ and $\mbf{X}_{c_2} = \{X_{m_c + 1}, \dots, X_b\}$ for some $a\!\leq\!m_c\!\leq\!b$.

\begin{align*}
&\boxed{p_{q}(\mbf{X}_{a:b})\cdot p_{\mscr{A}}(\mbf{X}_{a:{b+1}} \given Z_a\!=\!i, Z_{b+1}\!=\!j)} \\
&=\sum_{c \in \textsc{Ch}(q)} p_c(\mbf{X}_{a:b}) \cdot w_{qc} \cdot p_{\mscr{A}}(\mbf{X}_{a:{b+1}} \given Z_a\!=\!i, Z_{b+1}\!=\!j)\\
&=\sum_{c \in \textsc{Ch}(q)} p_{c_1}(\mbf{X}_{a:m_c}) \cdot p_{c_2}(\mbf{X}_{m_c+1:{b+1}})\cdot w_{qc} \cdot p_{\mscr{A}}(\mbf{X}_{a:{b+1}} \given  Z_a\!=\!i, Z_{b+1}\!=\!j)\\
&=\sum_{c \in \textsc{Ch}(q)} p_{c_1}(\mbf{X}_{a:m_c}) \cdot p_{c_2}(\mbf{X}_{m_c+1:b+1})\cdot w_{qc} \cdot \sum_{k} p_{\mscr{A}}(\mbf{X}_{a:{b+1}}, {Z_{m_c}}\!=\!k \given  Z_a\!=\!i, Z_{b+1}\!=\!j)\\
&=\sum_{c \in \textsc{Ch}(q)} \sum_{k} w_{qc} \cdot p_{\mscr{A}}(Z_{m_c+1}\!=\!k \given  Z_a\!=\!i, Z_{b+1}\!=\!j) \\
&\quad\cdot p_{c_1}(\mbf{X}_{a:m_c}) \cdot p_{\mscr{A}}(\mbf{X}_{a:m_c} \given Z_{m_c+1}\!=\!k,  Z_a\!=\!i, Z_{b+1}\!=\!j) \cdot p_{c_2}(\mbf{X}_{m_c+1:b}) \cdot p_{\mscr{A}}(\mbf{X}_{a:b} \given Z_{m_c+1}\!=\!k,  Z_a\!=\!i, Z_{b+1}\!=\!j)\\
&=\sum_{c \in \textsc{Ch}(q)} \sum_{k} w_{qc} \cdot p_{\mscr{A}}(Z_{m_c+1}\!=\!k \given Z_a\!=\!i, Z_{b+1}\!=\!j) \\
&\quad\quad\cdot \boxed{p_{c_1}(\mbf{X}_{a:m_c}) \cdot p_{\mscr{A}}(\mbf{X}_{a:m_c} \given Z_a\!=\!i, Z_{m_c+1}\!=\!k)} \cdot \boxed{p_{c_2}(\mbf{X}_{m_c+1:b}) \cdot p_{\mscr{A}}(\mbf{X}_{a:b} \given Z_{m_c+1}\!=\!k, Z_{b+1}\!=\!j)}
\end{align*}
Now let's analyze the complexity of the constructed circuit, which we denote by $\mscr{C}$. $\mscr{C}$ has $O(mkh^2)$ $\oplus$ nodes in total, where $m$ is the number of scopes in $\mscr{B}$, $k := \max_{\mbf{S} \text{ a scope in $\mscr{B}$}}|\{\oplus \in \mscr{B} \text{ with scope } \mbf{S}\}|$, and $h$ is the hidden states size of $\mscr{A}$. Suppose that each $\oplus$ node in $\mscr{B}$ has at most $r$ children, then each $\oplus$ node in $\mscr{C}$ has at most $O(rh)$ children. Hence the size of $\mscr{C}$ is bounded by $O(mkh^2 \cdot rh) = O(mkr \cdot h^3)$. Note that $O(mkr)$ corresponds to $O(|\mscr{B}|)$ and $O(h^3)$ is upper-bounded by $O(h^4)$, which is $O(|\mscr{A}|^2)$; hence the size of $\mscr{C}$ is bounded by $O(|\mscr{A}|^2 |\mscr{B}|)$, which is the same complexity as stated in Theorem~\ref{theorem:linear_structured}. Hence, we can remove the assumption that $\mscr{B}$ has to be structured from Theorem~\ref{theorem:linear_structured}.
\thmContiguousNonstructured*
By a similar recursive construction, we can also remove the assumption that $\mscr{B}$ has to be structured from Theorem~\ref{theorem:depth_structured}:
\begin{theorem}
\label{theorem:depth_contiguous}
    Let $\mscr{A}$ and $\mscr{B}$ be contiguous PCs. If $\mscr{A}$ is structured of depth $d$, then then we can construct a product circuit of $\mscr{A}$ and $\mscr{B}$ of size bounded by $O(|\mscr{A}|^{12d}|\mscr{B}|)$.
\end{theorem}

\section{COMPLETE EXAMPLE OF RESTRUCTURING}

In the section, we show a simple example of our restructuring algorithm in its entirety. In particular, we will consider the restructuring of a 4-variable circuit following a linear vtree, to a contiguous log-depth vtree. 

In Figure \ref{fig:hmm_4} we show the procedure of converting the linear vtree into a Bayesian network representation, the target contiguous vtree $W$, and the labeling constructed by Algorithm \ref{alg:compute_label}. In particular, this is a valid cover according to Definition \ref{def:correct_labels} as $\{Z_2\}$ is a cover for both $\{X_1, X_2\}$ and $\{X_3, X_4\}$, and the second and third conditions in the definition are satisfied as e.g. $\{Z_2\}$ blocks paths between $\{X_1, X_2\}$ and $\{Z_2\}$ (trivially).

We further show an explicit example of the restructuring of a PC respecting the right-linear vtree in Figure \ref{fig:hmm_4_pc}. In the original PC, we have each $Z_i \in \{0, 1\}$, corresponding to one of the product nodes in each scope. In the restructuring, we follow the generic construction in Figure \ref{fig:construction} for each triple of vtree nodes $(w, l, r)$ given the labeling in Figure \ref{fig:hmm_4_label}, giving rise to the restructured PC. For example, note that the root probabilities $(0.66, 0.34)$ correspond to the marginal probability $p(Z_2)$ in the original PC. One subtlety is that the leaf nodes do not all directly correspond to the leaf nodes in the original PC; for example $p(X_1|Z_2) = \sum_{Z_1} p(Z_1|Z_2) p(X_1|Z_1)$ where $p(X_1|Z_1)$ are the leaf nodes with scope $\{X_1\}$ in the original circuit; but this is not problematic as we can simply mix the leaf nodes (either explicitly with a sum node, or by constructing new leaf nodes). 

\begin{figure}
    \centering
    \begin{subfigure}{0.3\linewidth}
        \centering
        \includegraphics[width=0.9\linewidth]{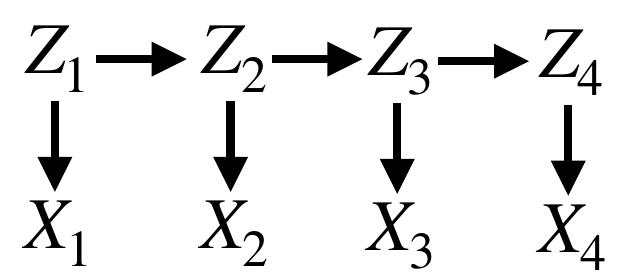}
        \caption{Bayesian Network for linear vtree $V_{v\mapsto Z_v}$}
    \end{subfigure}
    \begin{subfigure}{0.3\linewidth}
        \centering
        \includegraphics[width=0.8\linewidth]{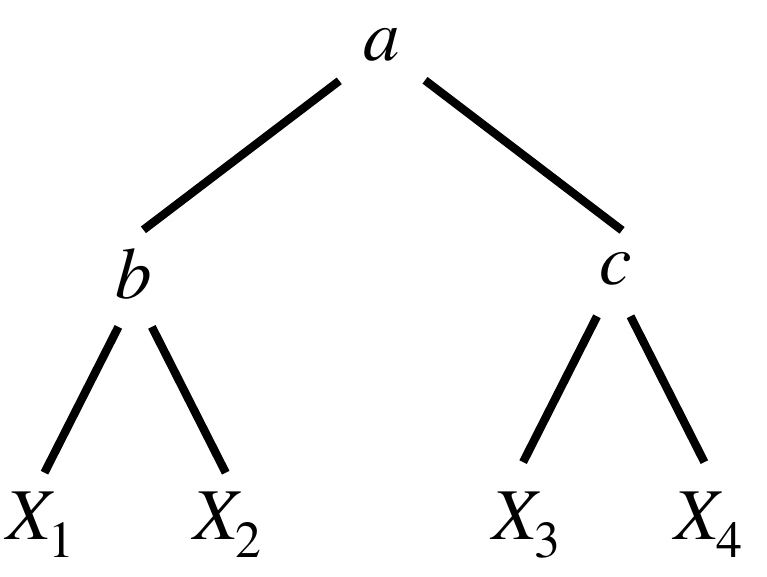}
        \caption{Target vtree $W$}
    \end{subfigure}
    \begin{subfigure}{0.37\linewidth}
        \centering 
        \includegraphics[width=0.8\linewidth]{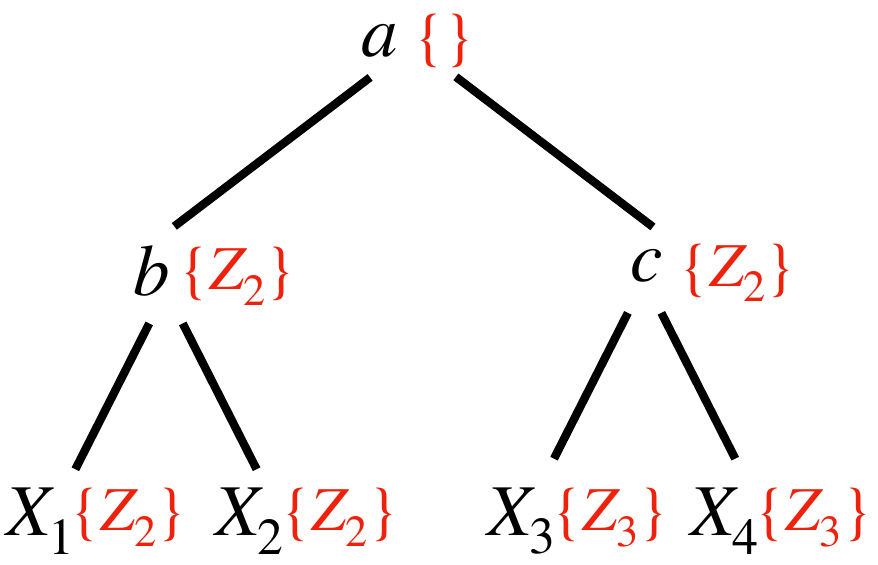}
        \caption{Labelled target vtree $W$}
        \label{fig:hmm_4_label}
    \end{subfigure}
    \caption{Example of restructuring labeling.}
    \label{fig:hmm_4}
\end{figure}

\begin{figure}
    \centering
    \begin{subfigure}{0.3\linewidth}
        \centering
        \includegraphics[width=0.95\linewidth]{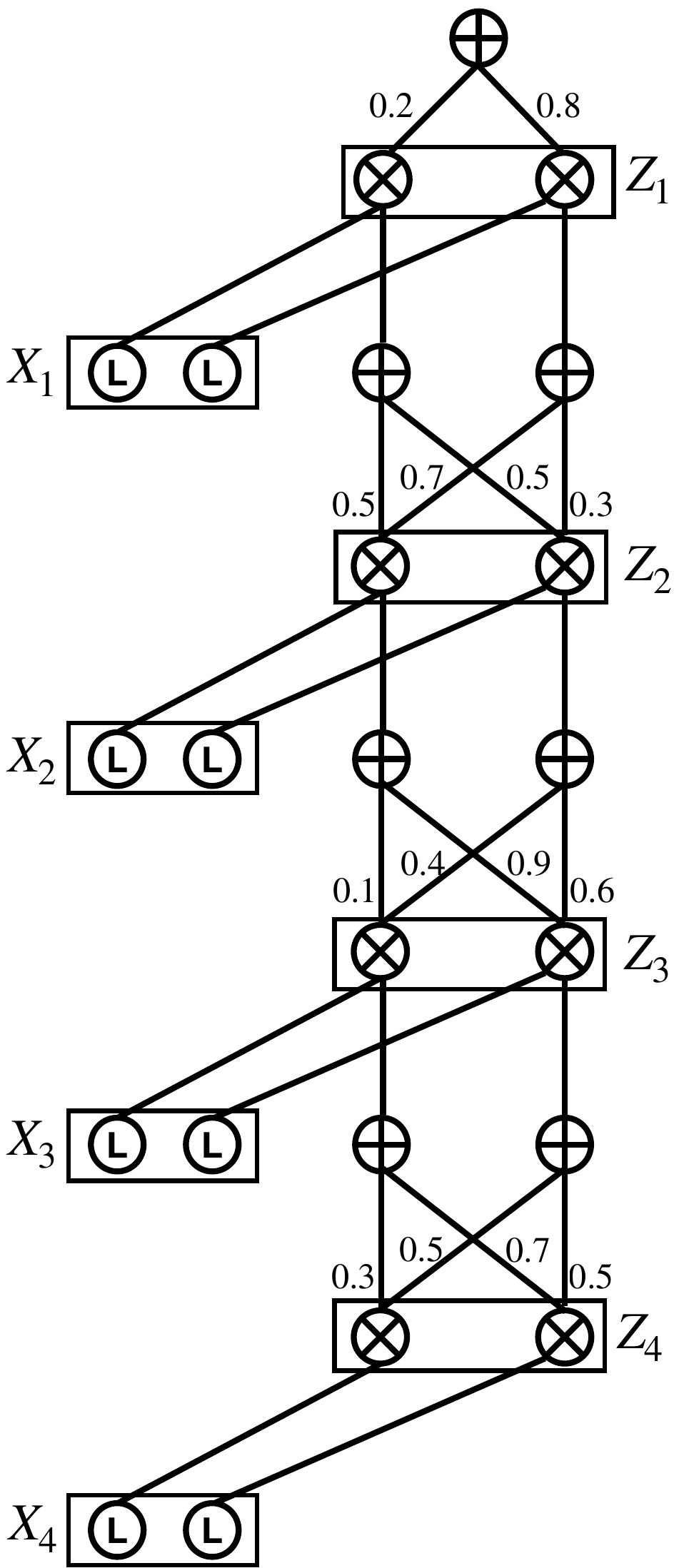}
        \caption{Original PC}
    \end{subfigure}
    \begin{subfigure}{0.69\linewidth}
        \centering
        \includegraphics[width=0.95\linewidth]{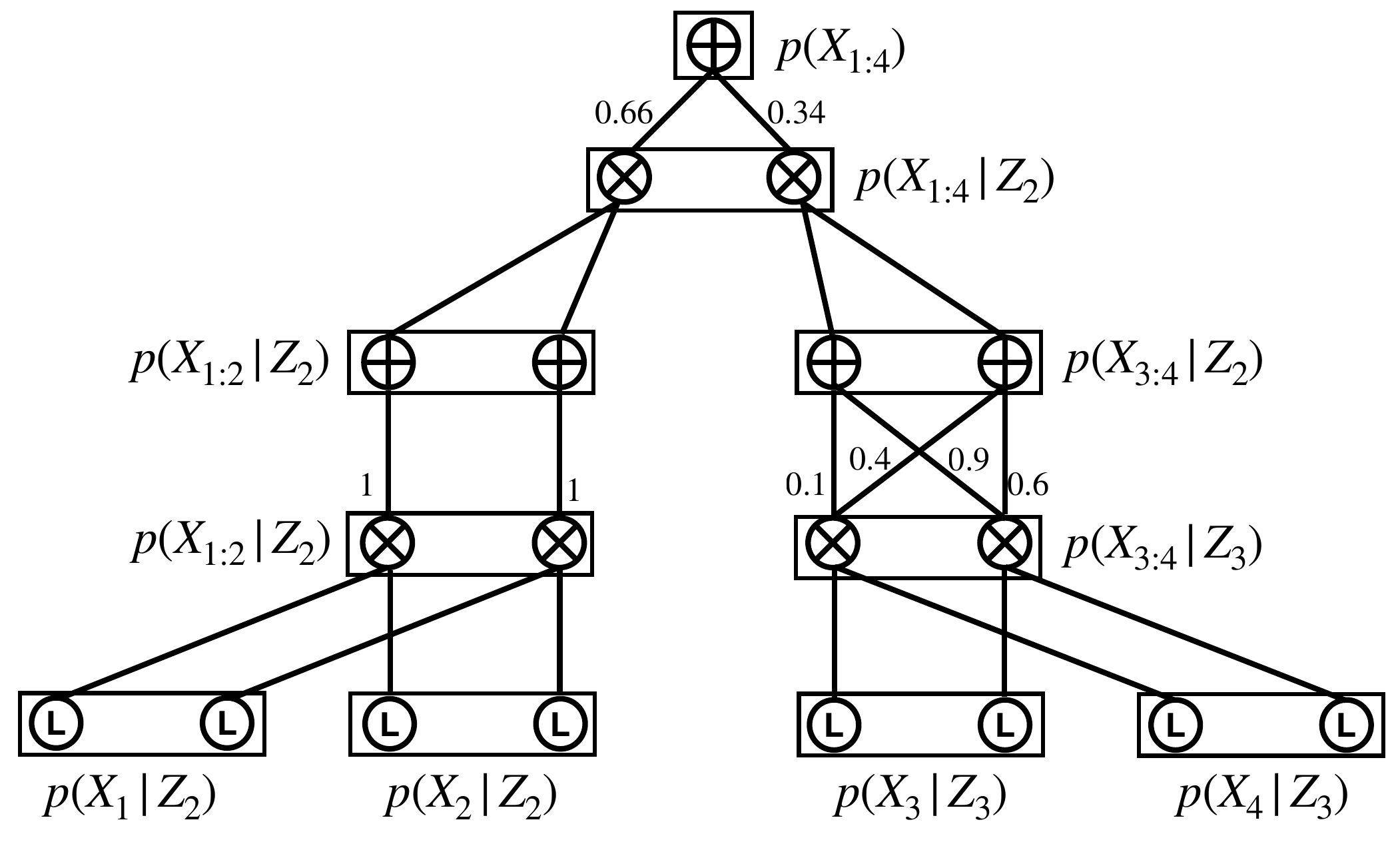}
        \caption{Restructured PC}
    \end{subfigure}
    \caption{Example of original and restructured PC for the structures in Figure \ref{fig:hmm_4}.}
    \label{fig:hmm_4_pc}
\end{figure}

\end{document}